\documentclass{article}

 \usepackage[preprint]{neurips_2026}

\usepackage{graphicx}
\usepackage[utf8]{inputenc} 
\usepackage[T1]{fontenc}    
\usepackage{hyperref}       
\usepackage{url}            
\usepackage{multirow}
\usepackage{booktabs}       
\usepackage{amsfonts}       
\usepackage{nicefrac}       
\usepackage{microtype}      
\usepackage{xcolor}         
\usepackage{subcaption}
\usepackage{amsmath,amsthm,amssymb,amsfonts}
\usepackage{bm}
\usepackage{algorithm}
\usepackage{algorithmic}
\usepackage{wrapfig}
\newtheorem{theorem}{Theorem}
\newtheorem{definition}{Definition}
\newtheorem{lemma}{Lemma}
\newtheorem{assumption}{Assumption}

\newtheorem{remark}{Remark}

\title{Bayesian Optimization with Structured Measurements: A Vector-Valued RKHS Framework}

\author{%
  Wenbin Wang\thanks{Correspondence to \texttt{wenbin.wang@epfl.ch}.} \\
  Automatic Control Laboratory, EPFL\\
  \And
  Colin N. Jones \\
  Automatic Control Laboratory, EPFL\\
}

\begin{document}

\maketitle

\begin{abstract}

Bayesian optimization (BO) is an efficient framework for optimizing expensive black-box functions. However, it is typically formulated as learning an end-to-end mapping from inputs to scalar objectives, thereby discarding the potentially rich information whenever a structured system output is available. In this work, we study Bayesian optimization over a vector-valued operator with structured measurements,  where each measurement observes multidimensional or functional outputs, e.g., trajectories or spatial fields, rather than a single scalar value. The objective is then defined as a linear functional of these measurements. This allows each observation to reveal substantially richer information about the underlying system compared to scalar observations. Assuming the unknown operator lies in a vector-valued reproducing kernel Hilbert space (RKHS), we derive high-probability concentration bounds for the kernel ridge regression (KRR) estimator directly in the measurement space, characterizing uncertainty in a general Hilbert space. Building on these results, we propose an algorithm based on the upper confidence bound (UCB) acquisition function with regret guarantees under mild assumptions, recovering sublinear rates for common kernels. Empirically, we demonstrate that leveraging structured measurements leads to improved sample efficiency by enabling efficient transfer of information across objectives and adaptation to time-varying settings.
\end{abstract}
\section{Introduction}



Bayesian optimization (BO) is an efficient sequential strategy for optimizing expensive black-box functions \cite{srinivas2012information,chowdhury2017kernelized,frazier2018tutorial,garnett2023bayesian}. It is commonly formulated as learning an end-to-end mapping from inputs to a scalar-valued objective using a Gaussian process surrogate model \cite{chowdhury2017kernelized, srinivas2012information}. However, in many real-world problems, this scalar objective is not directly generated by the system, but instead arises as a functional of an underlying structured output, such as a trajectory, field, or time series. For example, in building operation, the cost is defined by integrating heating prices over realized heating power trajectories \cite{shi2026disturbance,wang2025personalized}. Standard BO treats the problem as learning a direct mapping from inputs to scalar objective, effectively discarding the intermediate trajectory \cite{yu2025price,XU2024122493,menn2026preferential,hose2025fine,brunzema2022controller}. This limitation becomes particularly pronounced when multiple objectives correspond to different functionals of the same trajectory (e.g., different price signals), where existing methods must relearn each objective separately instead of sharing information through the common structure. While algorithms such as multi-task Bayesian optimization \cite{chowdhury2021no,evgeniou2004regularized,coutinho2026efficient,swersky2013multi} partially address this issue by modeling correlations across objectives, they typically rely on a finite number of scalar observations and a careful design of the correlation matrix. Another line of work models such correlations in the input space, such as contextual Bayesian optimization \cite{krause2011contextual,xu2024data,wang2025personalized}, which can be sample-inefficient for the exploration across different contexts. When structured outputs are available for measurement, each observation can reveal more information about the underlying system \cite{Yuan_2010}. This motivates leveraging such measurements in algorithm design, which can potentially improve learning efficiency.

To account for structured measurements, it is natural to consider Bayesian optimization with vector-valued outputs \cite{kadri2016operator,carmeli2010vector,micchelli2005learning}, where the black-box operator maps each input to outputs in a Hilbert space. This formulation enables modeling the underlying system through vector-valued observations. However, existing approaches remain limited. Most work focuses on finite-dimensional Euclidean output spaces \cite{Kudva_2026,chowdhury2021no,swersky2013multi,sessa2023multitask}, whereas many practical problems involve vector-valued trajectories and functions. Another line of work considers infinite-dimensional outputs but assumes that measurements are available only through linear functionals \cite{brault2019infinite,huang2026function,lahr2026optimal}, effectively reducing each measurement to a scalar value. As a result, these methods fail to exploit the full structure of the underlying output and limit the amount of information obtained from each observation.

In this work, we study Bayesian optimization for black-box systems that produce structured, vector-valued outputs. We introduce a general framework in which the unknown operator lies in a vector-valued reproducing kernel Hilbert space (RKHS) and is observed through a known linear measurement operator, possibly mapping into a different Hilbert space. The objective is then defined as a linear functional of these measurements. This formulation unifies a broad class of existing problem settings \cite{brault2019infinite,huang2026function,chowdhury2021no,vien2018bayesian,lahr2026optimal,Kudva_2026,swersky2013multi,sessa2023multitask} and enables structured measurements to be incorporated in a principled manner. Under mild assumptions, we establish high-probability concentration bounds for the vector-valued kernel ridge regression (KRR) estimator under general measurement operators. From the derived uncertainty bound, we propose an algorithm based on the upper confidence bound (UCB) of the estimator and prove sublinear regret guarantees for commonly used kernels. Numerical experiments demonstrate that structured measurements significantly improve sample efficiency, permitting information transfer across objectives and adaptation in time-varying settings. 

\textbf{Contributions.} The contributions of this work are as follows:
\begin{enumerate}
    \item \textbf{Bayesian optimization with structured measurements.}
    We introduce a general formulation of Bayesian optimization in which measurements are obtained through linear measurement operators acting on a latent nonlinear vector-valued operator. This enables richer information to be incorporated at each observation and naturally provides the potential for improving learning efficiency.

    \item \textbf{Theoretical guarantees under general measurement operators.}
    We establish high-probability concentration bounds for vector-valued KRR under arbitrary bounded linear measurement operators in a general Hilbert space. Based on these results, we propose a UCB-type algorithm and prove sublinear regret guarantees, extending classical scalar-valued BO with structured observations.

    \item \textbf{Improved efficiency in real-world applications.}
    Through synthetic benchmarks and real-world experiments, we demonstrate that the proposed method significantly improves sample efficiency by facilitating efficient information transfer across different objectives and achieving superior performance in optimizing time-varying objectives.
\end{enumerate}

\subsection{Related work}

\textbf{Learning nonlinear vector-valued operators.} Learning vector-valued or function-valued operators has been widely studied in the functional data analysis (FDA) literature \cite{gertheiss2024functional,ramsay1997functional}. Classical approaches typically rely on discretization, representing the outputs in finite-dimensional Euclidean spaces \cite{kadri2010nonlinear,kadri2009general,li2020neural}. Some recent work models functions directly in Hilbert spaces \cite{brault2019infinite,kadri2016operator}, using tools such as multi-output Gaussian processes \cite{bonilla2007multi,ruiz2024survey} and neural operators \cite{kovachki2023neural}. While these methods enable flexible modeling of vector-valued operators, they primarily focus on supervised learning and do not address sequential optimization problems. \textbf{BO with vector-valued outputs.}
Several works extend Bayesian optimization to structured or multi-output settings. Multi-task Bayesian optimization \cite{chowdhury2021no,sessa2023multitask,swersky2013multi} leverages correlations across multiple outputs to improve sample efficiency, while multi-objective Bayesian optimization \cite{khan2002multi,daulton2022multi} focuses on learning Pareto-optimal trade-offs. However, these approaches typically assume a finite number of scalar outputs and do not consider measurements such as functions or trajectories. \textbf{BO with augmented inputs.}
A related line of work assumes that observations are obtained through linear functionals of the vector-valued output \cite{huang2026function,lahr2026optimal}. They view the linear functional as the augmented input of a scalar-valued function with an augmented input space, which reduces the problem to a scalar-valued setting. Consequently, this formulation cannot quantify the uncertainty of the vector-valued operator. A comprehensive analysis of uncertainty quantification under general linear measurement operators is still missing.

\section{Problem formulation}
\label{section_2}

\paragraph{Problem formulation.} We consider the problem of optimizing a black-box system in which, for each input, the objective is defined as a linear functional of structured measurements of an underlying vector-valued operator. Specifically, we consider
\begin{equation}
\label{eq_obj_linearfunctional}
\max_{x \in \mathcal{X}} \quad F(x) = \langle m, M f(x) \rangle_{\mathcal{M}},
\end{equation}
where $f : \mathcal{X} \to \mathcal{Y}$ is an unknown vector-valued operator, $\mathcal{Y}$ is the output Hilbert space. The operator $M : \mathcal{Y} \to \mathcal{M}$ is a known bounded linear measurement operator, where $\mathcal{M}$ is the Hilbert space of measurements, and $m \in \mathcal{M}$ defines a linear functional that evaluates the measured output. For each input $x$, the output $f(x)$ is not directly observable; instead, only its measurement $M f(x)$ is available.

\begin{wrapfigure}{r}{0.5\textwidth}
    \centering
    \includegraphics[width=\linewidth]{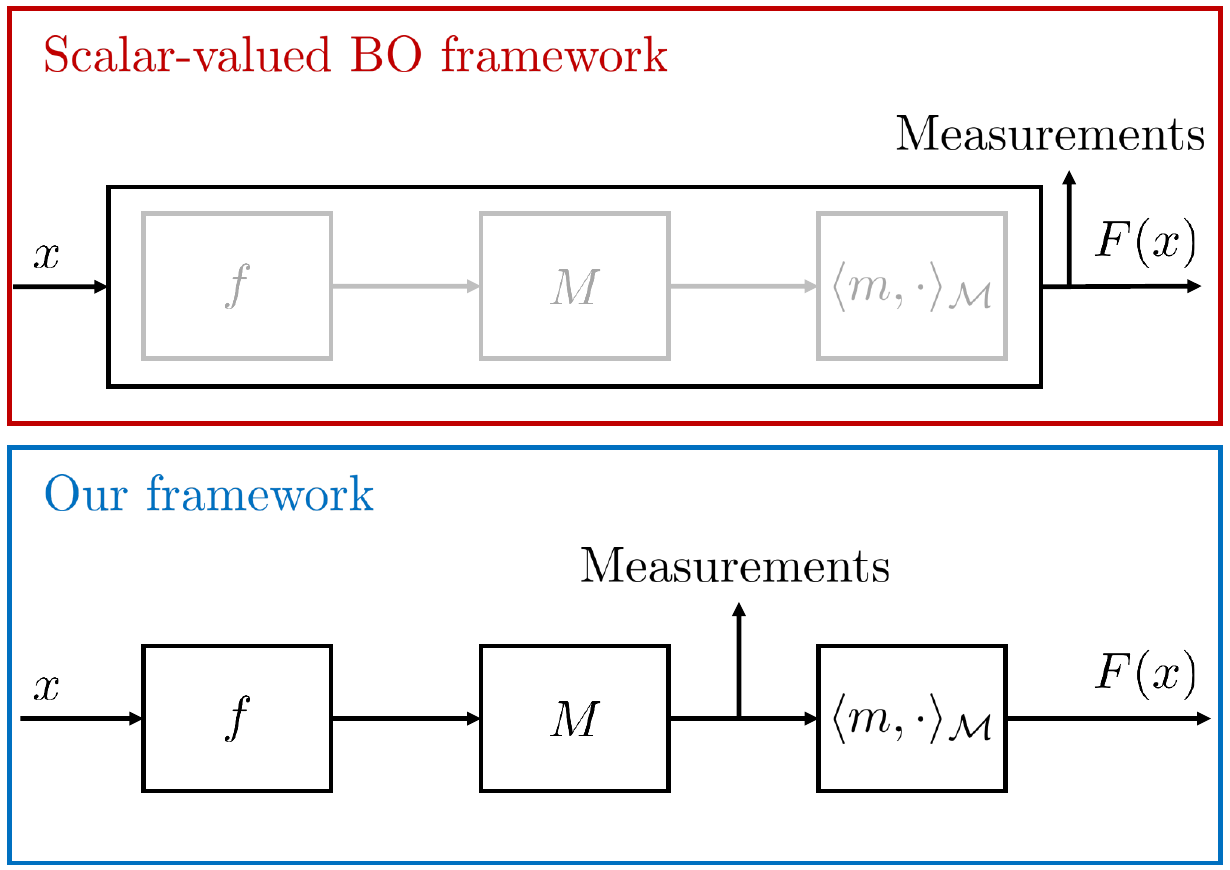}
    \caption{Vector-valued BO framework. Red box: classical BO framework. Blue box: vector-valued BO framework introduced in this work.}
    \label{fig:wrap}
\end{wrapfigure}
\textbf{Practical implication.} This formulation generalizes different observation regimes, ranging from full observation ($M=I$) to scalar observation ($M$ is a linear functional). The inner product $\langle m, M f(x) \rangle_{\mathcal{M}}$ naturally arises in many applications where the objective depends on vector-valued outputs linearly, e.g., energy consumption computed from trajectories in building control~\cite{wang2025personalized,shi2026disturbance}. It captures a broad class of common operations such as point evaluations, integrals, and averaging. Examples of common linear functionals in different Hilbert spaces are provided in Appendix~\ref{appendix_linear_functionals}. 

\textbf{Comparison to scalar-valued BO.}
As illustrated in Fig.~\ref{fig:wrap}, scalar-valued BO models an end-to-end mapping from $x$ to $F(x)$ using scalar feedback, thereby discarding the underlying structure in $f(x)$. This limits the ability to exploit richer observations and prevents information transfer across different objectives. In contrast, our formulation explicitly incorporates structured measurements $Mf(x)$, enabling more informative learning and adaptation across objectives.

\textbf{Comparison to BO with augmented input.} Observe that for any $m$ and $M$, $\langle m, M f(x) \rangle_{\mathcal{M}} 
= \langle M^* m, f(x) \rangle_{\mathcal{Y}}$. Thus, the objective can be viewed as applying a linear functional $M^* m \in \mathcal{Y}$ to $f(x)$. A common approach is to model this quantity using a scalar-valued function $\tilde{F}(x,y) \mapsto \langle y, f(x) \rangle_{\mathcal{Y}}$ with an augmented input space $\mathcal{X} \times \mathcal{Y}$, which is learned via a scalar kernel~\cite{lahr2026optimal,huang2026function}. While this enables standard scalar-valued analysis, it reduces the
observation to a scalar projection of the output. In contrast, our formulation provides flexibility. Different choices of $M$ correspond to different observation mechanisms, which can significantly affect the amount of information obtained on $f(x)$.

\paragraph{Performance metric.}
Since $f$ is a black-box operator, solving \eqref{eq_obj_linearfunctional} in a single step is generally intractable. Instead, we aim to construct a sequence of inputs $\{x_t\}_{t=1}^T$ under a finite budget $T$ such that the resulting averaged performance approximates the optimal solution of \eqref{eq_obj_linearfunctional}. Specifically, at each iteration $t$, the learner queries $f$ at a chosen point $x_t$ and observes noisy measurements $y_t = M f(x_t) + \delta_t$, where $\delta_t \in \mathcal{M}$ is additive noise. The goal is to design a policy that balances exploration and exploitation, which can be quantified via the cumulative regret
\[
\mathcal{R}_T = \sum_{t=1}^T \langle m, Mf(x^*) - Mf(x_t) \rangle_{\mathcal{M}},
\]
where $x^*$ denotes an optimal solution to \eqref{eq_obj_linearfunctional}. Sublinear growth of $\mathcal{R}_T$ with respect to $T$ indicates the problem is effectively solved \cite{srinivas2012information,chowdhury2017kernelized}. In general, such results are not attainable for arbitrary operators and noise processes. To obtain meaningful results, we impose the following assumptions.
\begin{assumption}[Observation model]
\label{assumption_noise}
    Conditioned on the natural filtration $\mathcal{F}_{t-1}$, $\delta_t\in \mathcal{M}$ is a zero-mean sub-Gaussian random element, i.e., $\exists \sigma>0$, $\forall \alpha \in\mathcal{M}$, $\mathbb{E}[e^{\langle \alpha,\delta_t\rangle_{\mathcal{M}}}|\mathcal{F}_{t-1}]\leq e^{\sigma^2\|\alpha\|_{\mathcal{M}}^2/2}$.
\end{assumption}
\begin{remark}
The sub-Gaussian noise assumption is commonly adopted in the literature \cite{chowdhury2021no,huang2026function}. It ensures that the noise has controlled (light-tailed) behavior. Depending on the measurement space $\mathcal{M}$, the noise $\delta$ can take values in an infinite-dimensional Hilbert space. The sub-Gaussian property implies that the noise distribution cannot spread uniformly across infinitely many directions. Instead, its magnitude must decay sufficiently fast along higher-order directions.
\end{remark}

\paragraph{Vector-valued RKHS.} Vector-valued RKHS is a direct generalization of scalar-valued RKHS \cite{micchelli2005learning} and contains the operators with outputs in Hilbert space $\mathcal{Y}$. We adopt its definition from \cite{micchelli2005learning}.
\begin{definition}
    The Hilbert space $\mathcal{H}$ is a reproducing kernel Hilbert space if for any $y\in\mathcal{Y}$ and $x\in \mathcal{X}$, the linear functional that maps $f\in\mathcal{H}$ to $\langle y,f(x)\rangle_{\mathcal{Y}}$ is continuous.
\end{definition}
According to the Riesz representation theorem~\cite{akhiezer1981theory}, we can define a mapping $\Phi:\mathcal{X}\rightarrow\mathcal{L}(\mathcal{Y},\mathcal{H})$, where $\mathcal{L}(\mathcal{Y},\mathcal{H})$ is the Banach space of bounded linear operator, such that for all $y\in \mathcal{Y}$, $\langle y,f(x)\rangle_{\mathcal{Y}} = \langle \Phi(x)y,f\rangle_{\mathcal{H}}$. Consequently, a mapping $K:\mathcal{X}\times\mathcal{X}\rightarrow\mathcal{L}(\mathcal{Y})$ defined by $K(x,s)y = (\Phi(x)y)(s)$ is called a vector-valued kernel. Further properties of such a kernel can be found in \cite{caponnetto2008universal,carmeli2010vector}.



\begin{assumption}[Regularity]
    \label{assumption_regularity}
    The unknown operator $f$ belongs to the vector-valued RKHS $\mathcal{H}$, induced by the operator-valued kernel $K$. The RKHS norm of $f$ is bounded as $\|f\|_{\mathcal{H}} \leq \Gamma$. Furthermore, we assume that the kernel is uniformly bounded, i.e.,
    \[
    \sup_{x \in \mathcal{X}} \|K(x,x)\|_{\mathrm{op}} \leq C_K.
    \]
    We further assume that given a pairwise distinct input set $X_t:=\{x_i\}_{i = 1}^t$, any operator-valued matrix $K_{X_tX_t}$ with its elements given by $(K_{X_tX_t})_{i,j} = K(x_i,x_j)$ is a trace-class operator, i.e. the sum of eigenvalues (point spectrum) of $K_{X_tX_t}$ is finite.
\end{assumption}

\begin{remark}
The boundedness of the kernel and the RKHS norm ensures that the target operators lie in a well-structured hypothesis space, enabling efficient uncertainty reduction through sequential sampling. The trace-class assumption on the operator-valued kernel is essential for deriving theoretical guarantees, as it ensures that the information gained from the samples is finite. This condition is mild and satisfied by a broad class of kernels, such as separable kernels \cite{huang2026function,chowdhury2021no}.
\end{remark}


\paragraph{Frequentist perspective} In this work, we adopt a frequentist perspective and assume that the unknown operator is fixed with bounded norm. The only source of randomness arises from the observation noise. Accordingly, we analyze worst-case regret over all possible operators, whereas the Bayesian formulation considers expected regret under a prior distribution.

\textbf{Extension to multiple and time-varying objectives.}
While we focus on the formulation in \eqref{eq_obj_linearfunctional} with a fixed objective for clarity of analysis, the proposed framework naturally extends to settings where the objective varies across iterations, e.g., $F_t(x) = \langle m_t, M f(x) \rangle$. In particular, by quantifying the structured measurement $Mf(x)$, the resulting confidence sets can be transferred across different objectives defined by different linear functionals. This enables efficient adaptation to multiple objectives and time-varying settings. We defer a detailed discussion to Appendix \ref{appendix_extension}.
\section{Confidence bound in measurement space}
\paragraph{Vector-valued KRR.}
Let $X_t := \{x_i\}_{i=1}^t$ be a set of pairwise distinct inputs and $Y_t := (y_1,\dots,y_t)\in \mathcal{M}^t$ the corresponding noisy observations, where $y_i = M f(x_i) + \delta_i$, and $\mathcal{M}^t$ is the direct sum of $t$ copies of $\mathcal{M}$. An estimator $\mu_t\in\mathcal{H}$ of $f$ is obtained by solving the regularized risk minimization problem
\begin{equation}
\label{eq_KRR}
\mu_t 
= \arg\min_{\tilde{f} \in \mathcal{H}} 
\sum_{i=1}^{t} \|M\tilde{f}(x_i) - y_i\|_{\mathcal{M}}^2 
+ \lambda \|\tilde{f}\|_{\mathcal{H}}^2,
\end{equation}
where $\lambda > 0$ is a regularization parameter controlling the complexity of the estimator. The explicit formulation of $\mu_t$ is provided in Appendix~\ref{appendix_lemma_representer_proof}. To analyze the prediction error $\|Mf(x) - M\mu_t(x)\|_{\mathcal{M}}$, a natural approach is to first bound the error in the output space $\mathcal{Y}$ and then propagate it through the measurement operator $M$. However, this can be overly conservative, as it requires controlling errors in all directions of $\mathcal{Y}$, including those that are not observable through $M$. This observation motivates modeling the measurable quantity $Mf(x)$ directly in the measurement space.

\paragraph{Induced kernel in the measurement space.}
We introduce the induced kernel $K^{M} : \mathcal{X}\times\mathcal{X} \rightarrow \mathcal{L}(\mathcal{M})$ defined by
\[
K^{M}(x,s) = M K(x,s) M^*,
\]
which is an operator-valued kernel taking values in $\mathcal{L}(\mathcal{M})$. The following lemma shows that $K^M$ is a reproducing kernel in the measurement space.

\begin{lemma}[Proposition 6, \cite{carmeli2010vector}]
\label{lemma_m}
Let $M \in \mathcal{L}(\mathcal{Y},\mathcal{M})$ and $K^{M}(x,s) = M K(x,s) M^*$. Then $K^{M}$ is a reproducing kernel, and the corresponding RKHS $\mathcal{H}_{M}$ is continuously embedded in $\mathcal{H}$.
\end{lemma}

Lemma~\ref{lemma_m} shows that composing elements of a vector-valued RKHS with a bounded linear operator yields a new vector-valued RKHS with kernel $K^{M}$.  Consequently, the induced estimator in the measurement space can equivalently be obtained by performing KRR in $\mathcal{H}_M$ using the observations $Y_t$ \cite{kadri2016operator,micchelli2005learning}. Specifically, for any $x\in\mathcal{X}$, the estimator of $Mf(x)$ in $\mathcal{M}$ is given by
\begin{equation}
\label{eq_krr_estimator}
    K_{xX_{t}}^{M} \bigl(K_{X_{t}X_{t}}^{M} + \lambda I\bigr)^{-1} Y_{t},
\end{equation}
where $K_{xX_{t}}^{M} \in \mathcal{L}(\mathcal{M}^t,\mathcal{M})$ is defined by $K_{xX_{t}}^{M} v = \sum_{i=1}^{t} K^{M}(x,x_i) v_i, \forall v=(v_1,\ldots,v_t)\in\mathcal M^t$, and $K_{X_{t}X_{t}}^{M} \in \mathcal{L}(\mathcal{M}^t,\mathcal{M}^t)$ is the block operator matrix with entries $(K_{X_{t}X_{t}}^{M})_{ij} = K^{M}(x_i,x_j)$. Note that $K^M_{X_{t}X_{t}}$ is a trace-class operator since $K_{X_{t}X_{t}}$ has finite trace. Importantly, \eqref{eq_krr_estimator} coincides with the measured output of the solution to \eqref{eq_KRR}, i.e., $M\mu_t(x)$ (see Appendix~\ref{appendix_lemma_representer_proof}). Thus, working in the measurement space does not change the estimator for the observable quantity $Mf(x)$, but provides a straightforward formulation. In particular, it enables uncertainty quantification directly in $\mathcal{M}$, avoiding the conservatism introduced by bounding errors in the full output space $\mathcal{Y}$.

\paragraph{Concentration bound.} To quantify the uncertainty of the estimator, we define
\[
K_{t}^{M}(x,x) 
= K^{M}(x,x) 
- K_{xX_{t}}^{M} \bigl(K_{X_{t}X_{t}}^{M} + \lambda I\bigr)^{-1} K_{X_{t}x}^{M},
\]
where $K_{X_{t}x}^{M}$ is the adjoint of $K_{xX_{t}}^{M}$. This operator can be interpreted as the posterior variance, which captures the remaining uncertainty at $x$ after observing $Y_{t}$.  Using this characterization, we can derive high-probability confidence bounds for the prediction error in the measurement space.

\begin{theorem}
\label{theorem_partial}
Let Assumptions \ref{assumption_noise}-\ref{assumption_regularity} hold. Assuming that $K_{X_{t}X_{t}}^M$ is a trace-class operator, given $X_{t}$ and $Y_{t}$, with probability at least $1-\zeta$, for any $x\in\mathcal{X}$, the prediction error is bounded by
\[
\|M f(x) \!-\! M \mu_{t}(x)\|_{\mathcal{M}}
\!\leq \!
\left(
\Gamma \!
+ \!\frac{\sigma}{\sqrt{\lambda}}
\sqrt{2 \log(1/\zeta)\! +\! \log \det\bigl(I\! +\! \lambda^{-1} K_{X_{t}X_{t}}^{M}\bigr)}
\right)
\|K_{t}^{M}(x,x)\|_{\mathrm{op}}^{1/2}.
\]
\end{theorem}

The proof of Theorem \ref{theorem_partial} can be found in Appendix \ref{appendix_theorem_partial_proof}. The bound consists of two components. The first term, $\log \det\bigl(I + \lambda^{-1} K_{X_{t}X_{t}}^M\bigr)$, resembles the mutual information in the scalar-valued setting. Assuming that $K_{X_{t}X_{t}}^M$ is a trace-class operator, i.e., the sum of its eigenvalues is finite, ensures that the Fredholm determinant is well-defined and finite. The second term, $\|K_{t}^M(x,x)\|_{\mathrm{op}}^{1/2}$, captures the uncertainty at the query point $x$. In practice, this quantity can be evaluated via the largest eigenvalue of $K_{t}^M(x,x)$. This upper bound exhibits a structure similar to the bounds in \cite{chowdhury2017kernelized,chowdhury2021no}. In fact, it can be viewed as a direct generalization of these results, extending the measurement space from a finite-dimensional space to a possibly infinite-dimensional Hilbert space.



\subsection{Special cases}
\label{section_two_cases}

\textbf{Full observation ($M = I$).} When $M = I$, the learner observes the full vector-valued output $f(x)$, and the induced kernel reduces to $K^M(x,s) = K(x,s)$. The objective becomes $\langle m, f(x)\rangle_{\mathcal{Y}}$ with $m \in \mathcal{Y}$. This corresponds to the most informative measurements, where the entire output is observed. In this case, the confidence bound in Theorem~\ref{theorem_partial} provides uncertainty quantification directly over the vector-valued output $f(x)$.

\textbf{Scalar linear measurements ($M: \mathcal{Y} \to \mathbb{R}$).} When $M$ is a bounded linear functional, the observations become scalar-valued. In this case, $K^M(x,s)$ reduces to a scalar kernel, and the problem recovers standard scalar-valued Bayesian optimization. This regime corresponds to minimal information acquisition, where each query reveals only a one-dimensional projection of $f(x)$.





These cases show that our framework unifies both full vector-valued observation and scalar-valued observation as special instances of a general formulation. Intermediate choices of $M$ interpolate between these two cases, providing a principled way to model the trade-off between observation richness and algorithm efficiency.

\section{Algorithm and regret analysis}

We construct an upper and lower bound for $F(x)$ at any iteration $t \geq 1$. In particular, given dataset $X_{t-1}$, $Y_{t-1}$, $F(x) \in [\underline{F}_{t-1}(x), \overline{F}_{t-1}(x)]$, where
\begin{align*}
\overline{F}_{t-1}(x)
&= \langle m, M \mu_{t-1}(x) \rangle_{\mathcal{M}}
+ \|m\|_{\mathcal{M}}\, \|M f(x) - M \mu_{t-1}(x)\|_{\mathcal{M}}, \\
\underline{F}_{t-1}(x)
&= \langle m, M \mu_{t-1}(x) \rangle_{\mathcal{M}}
- \|m\|_{\mathcal{M}}\, \|M f(x) - M \mu_{t-1}(x)\|_{\mathcal{M}}.
\end{align*}

The deviation term $\|M f(x) - M \mu_{t-1}(x)\|_{\mathcal{M}}$ can be bounded using Theorem~\ref{theorem_partial}, yielding a computable confidence interval. To solve problem \eqref{eq_obj_linearfunctional}, we adopt the principle of optimism in the face of uncertainty, leading to a UCB-based algorithm as described in Algorithm~\ref{alg:ucb_obj}.

\begin{algorithm}[t]
\caption{Vector-valued Bayesian optimization (vvBO)}
\label{alg:ucb_obj}
\begin{algorithmic}[1]
\REQUIRE Total budget $T$, regularization parameter $\lambda > 0$, confidence level $\zeta \in (0,1)$
\STATE Initialize dataset $X_0=\emptyset$, $Y_0=\emptyset$ with $M\mu_0(x) = 0$, $K_0^M(x,x) = K^M(x,x)$
\FOR{$t = 1,\dots,T$}
    \STATE Compute $M\mu_{t-1}(x)$, $K_{t-1}^M(x,x)$
    \STATE Choose the next sample via
    \[
    x_t \in \arg\max_{x \in \mathcal{X}} \ \langle m, M \mu_{t-1}(x)\rangle_{\mathcal{M}}
    + \beta_{t-1}\|m\|_{\mathcal{M}}\|K_{t-1}^{M}(x,x)\|_{\mathrm{op}}^{1/2}
    \]
    \STATE Query $f$ at $x_t$ and observe $
    y_t = M f(x_t) + \delta_t$
    \STATE Update the dataset $X_{t}$, $Y_{t}$
\ENDFOR
\end{algorithmic}
\end{algorithm}

\textbf{Algorithm.}
We are now ready to present the \textbf{vector-valued Bayesian Optimization (vvBO)} algorithm with structured measurements. We define the confidence parameter as $\beta_{t-1} = \Gamma + \frac{\sigma}{\sqrt{\lambda}}\sqrt{2\log(1/\zeta) + \log\det\bigl(I + \lambda^{-1}K_{X_{t-1}X_{t-1}}^{M}\bigr)}$. At each iteration $t$, the learner constructs a KRR estimator in the measurement space using the dataset $X_{t-1}, Y_{t-1}$, and forms an upper confidence bound based on Theorem~\ref{theorem_partial}. The next query point is selected as $x_t = \arg\max_{x \in \mathcal{X}} \ \langle m, M \mu_{t-1}(x) \rangle_{\mathcal{M}} + \beta_{t-1} \|m\|_{\mathcal{M}}\|K_{t-1}^M(x,x)\|_{\mathrm{op}}^{1/2}$, 
which balances exploration and exploitation through the confidence parameter $\beta_{t-1}$. The observation is then obtained as $y_t = M f(x_t) + \delta_t$. Notably, the entire procedure operates directly in the measurement space, leveraging structured observations through the induced kernel $K^M$.

\subsection{Regret analysis}
The following theorem establishes a general regret bound for Algorithm~\ref{alg:ucb_obj}.

\begin{theorem}
\label{theorem_regret}
Let Assumptions~\ref{assumption_noise}--\ref{assumption_regularity} hold. For any $\lambda > 0$ and $\zeta \in (0,1)$, with probability at least $1-\zeta$, the regret of Algorithm~\ref{alg:ucb_obj} satisfies
\begin{align*}
    \mathcal{R}_T\leq 2\|m\|_{\mathcal{M}}\left(\Gamma \!+ \!\frac{\sigma}{\sqrt{\lambda}}\sqrt{2\log(\frac{1}{\zeta})+ \log\det\bigl(I \!+ \!\lambda^{-1}K_{X_{T}X_{T}}^{M}\bigr)}\right)\sqrt{T\sum_{t = 1}^T\|K_{t-1}^M(x_t,x_t)\|_{\mathrm{op}}}.
\end{align*}
\end{theorem}
The proof of Theorem~\ref{theorem_regret} is deferred to Appendix~\ref{appendix_proof_theorem_2}. The resulting bound shares a similar structure with classical results in scalar-valued Bayesian optimization~\cite{chowdhury2017kernelized,srinivas2012information}, as well as extensions to finite-dimensional vector-valued settings \cite{chowdhury2021no}. Despite this similarity, establishing sublinear regret in our setting is nontrivial. The term $\log\det(I + \lambda^{-1}K_{X_{T}X_{T}}^M)$ involves an operator-valued kernel and is well-defined only when $K_{X_{T}X_{T}}^M$ is trace-class. However, this condition alone does not characterize how the quantity grows with $T$. Likewise, boundedness of $\|K_{t-1}^M(x_t,x_t)\|_{\mathrm{op}}$ does not control the cumulative growth of the uncertainty terms. These challenges motivate restricting attention to structured kernel classes, under which explicit bounds can be derived.

\textbf{Separable kernel.} We restrict attention to a widely used class of separable kernels \cite{carmeli2010vector}. Specifically, we consider kernels of the form $K = G \otimes B$, where $G:\mathcal{X}\times\mathcal{X}\rightarrow\mathbb{R}$ is a scalar-valued kernel and $B \in \mathcal{L}(\mathcal{Y})$ is a positive trace-class operator. Under a measurement operator $M$, this induces $K^M(x,s) = G(x,s)\, (M B M^*) \in \mathcal{L}(\mathcal{M})$. The trace-class property is preserved under the composition, i.e., $M B M^*$ remains trace-class for bounded $M$. We write $B_M = M B M^* \in\mathcal{L}(\mathcal{M})$ and  
\[
K_{X_{t}X_{t}}^M = G_{X_{t}X_{t}}\otimes B_M,\qquad K_{xX_{t}}^M = G_{xX_{t}}\otimes B_M, \qquad K_{X_{t}x}^M = G_{X_{t}x}\otimes B_M,
\]
where $G_{X_{t}X_{t}}\in \mathbb{R}^{t\times t}$ is the kernel Gram matrix and $G_{xX_{t}}\in \mathbb{R}^{1\times t}$, $G_{X_{t}x}\in \mathbb{R}^{t\times 1}$.

\begin{theorem}[Kernel specific bound]
\label{theorem_kernel_specific}
Let Assumptions \ref{assumption_noise}-\ref{assumption_regularity} hold. For the separable kernel with $B_M$ having finite spectrum, if $\mathcal{X}\subset \mathbb{R}^d$, with probability at least $1-\zeta$, the regret for Algorithm~\ref{alg:ucb_obj} satisfies:
\begin{enumerate}
    \item Linear kernel: $\mathcal{R}_T \leq \mathcal{O}\bigl(\log T\sqrt{T}\bigr)$,
    \item Gaussian kernel: $\mathcal{R}_T \leq \mathcal{O}\bigl((\log T)^{d+1}\sqrt{T})$,
    \item Mat\'ern kernel: $\mathcal{R}_T \leq \mathcal{O}\bigl(T^{d(d+1)/(2\nu+d(d+1))}\log T\sqrt{T}\bigr)$.
\end{enumerate}
\end{theorem}
The proof of Theorem~\ref{theorem_kernel_specific} is deferred to Appendix~\ref{appendix_theorem_3_proof}. 
The above results show that, with separable kernels, the proposed method achieves sublinear regret rates comparable to those in classical scalar-valued Bayesian optimization. In particular, the dependence on $T$ matches known results despite the system outputs being vector-valued elements in Hilbert spaces. These guarantees follow from combining the high-probability confidence bounds in Theorem~\ref{theorem_partial} with the information-gain bounds used in the proof of Theorem~\ref{theorem_kernel_specific}, demonstrating that efficient sequential optimization is achievable under general measurement operators.

\textbf{Extension to time-varying and nonlinear functions.} As discussed in Section~\ref{section_2}, the proposed algorithm readily extends to time-varying objectives with time-dependent linear functionals $m_t$. Details are provided in Appendix~\ref{appendix_extension}. Furthermore, the results in this work can also be naturally generalized to any nonlinear function $F':\mathcal{M}\rightarrow \mathbb{R}$ satisfying a Lipschitz condition. A detailed discussion is deferred to Appendix~\ref{appendix_extension}.

\textbf{Efficient computation.}
Note that for infinite-dimensional output spaces with general operator-valued kernels, operations such as inner products and norms can be computationally intractable. We provide explicit constructions for separable kernels in the appendix to enable the practical implementation of the proposed bounds and algorithm. The key idea is to exploit the spectral decomposition of the operator-valued kernel and truncate the resulting infinite series at a finite level. This yields a finite-dimensional approximation. A detailed formulation can be found in Appendix \ref{appendix_efficient_computation}.

\section{Numerical experiments}
\label{experiment}
In this section, we demonstrate that leveraging structured measurements improves efficiency in optimizing black-box systems with vector-valued outputs via several numerical experiments, particularly with changing objectives. We compare the proposed vector-valued Bayesian optimization (vvBO) algorithm against multiple baselines, including vanilla BO \cite{chowdhury2017kernelized}, multi-task BO (MTBO) \cite{chowdhury2021no}, function-on-function BO (FFBO) \cite{huang2026function}, and contextual BO (CTBO), where the context is defined as $M^* m$ in the augmented input space (see details in Appendix \ref{appendix_setup}). In all experiments, the underlying system produces structured outputs (e.g., trajectories or functions), and the objective is defined as a linear functional of these outputs. We consider both synthetic benchmarks and a real-world controller tuning task. All results are averaged over 10 independent runs with different initial starting points without pretraining. Simulation details, including baseline implementation, hyperparameters, and computational time, are provided in Appendix~\ref{appendix_setup}.

\textbf{Synthetic problems.}
Here, we explore how structured measurements help to adapt to different objectives. We consider the setting where $M = I$ for vvBO, i.e., measuring the full vector-valued outputs. Additional results for partial observations ($M \neq I$) are provided in Appendix~\ref{appendix_synthetic}. The optimization process is divided into three phases, each corresponding to a different objective defined by $m_i, i\in\{1,2,3\}$ (see detailed formulation in Appendix \ref{appendix_synthetic}). All the methods retain and reuse all previously collected samples across phases. Additionally, we also compare against vanilla BO (rBO) and multi-task BO (rMTBO) re-initialized when shifting to the next phase. Performance is evaluated using both simple and cumulative regret. The results for three test functions are shown in Fig.~\ref{fig:regret_test_function}. Across all experiments, vvBO with full vector-valued measurement consistently outperforms the baseline methods for different objectives. Since most baseline methods operate on scalar observations or fixed finite-dimensional representations, they cannot fully exploit the structure of the output or adjust effectively to changes in the objective. In contrast, our approach models the underlying vector-valued structure, facilitating effective reuse of information and rapid tracking. As for CTBO, while capable of handling varying objectives, it incurs larger cumulative regret since for each new context, it requires a large amount of exploration before establishing an accurate model. A detailed comparison between vvBO and CTBO is provided in Appendix~\ref{appendix_synthetic}.

\begin{figure}[t]
    \centering
    \includegraphics[width=\linewidth]{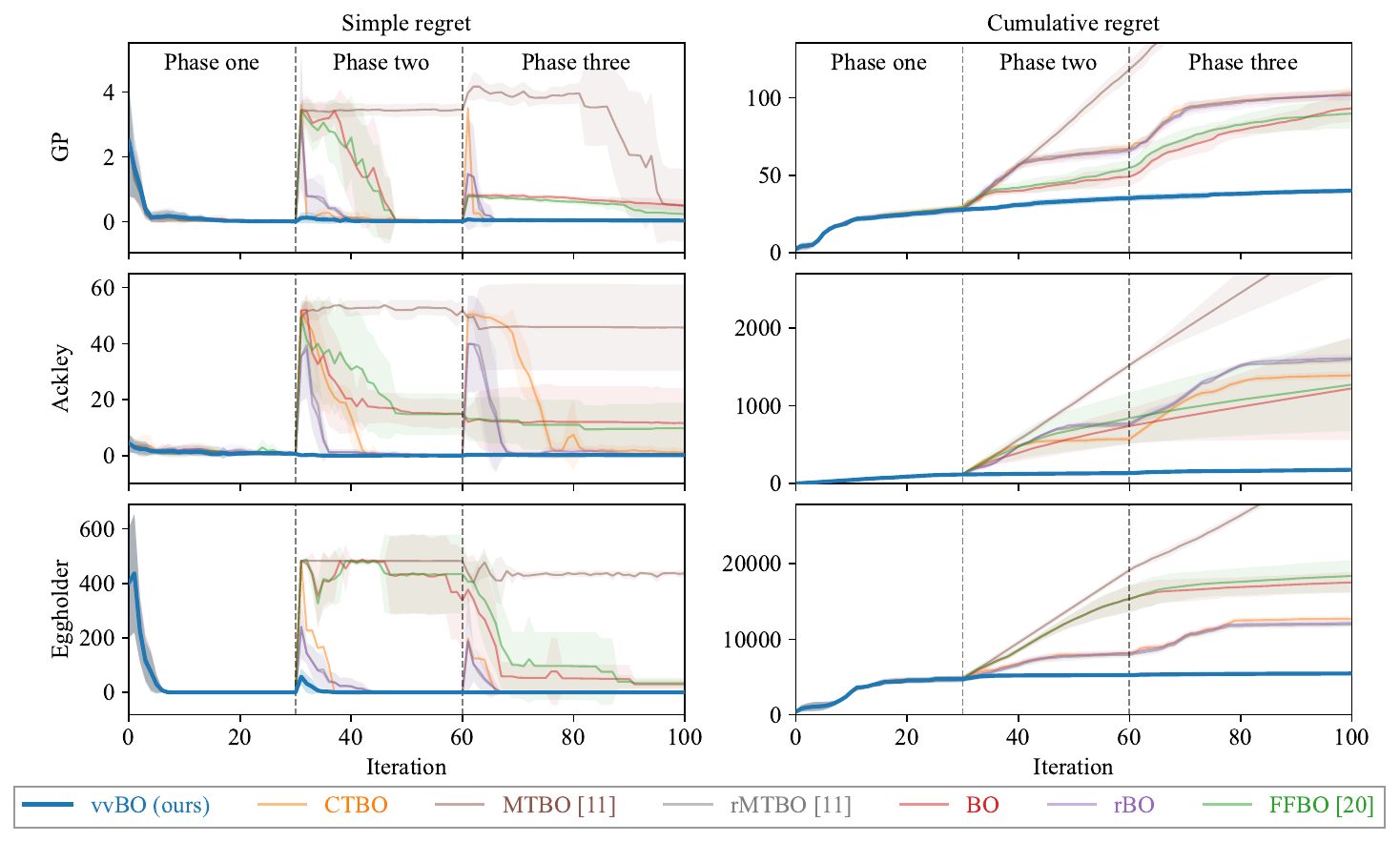}
    \caption{Comparison of simple regret and cumulative regret for three test operators with different baseline methods. vvBO with full trajectory measurements has the lowest regret in all scenarios, which indicates that it efficiently accommodates to different objectives.}
    \label{fig:regret_test_function}
\end{figure}

\textbf{Time-varying objectives in real-world controller tuning.}
We next evaluate vvBO in a more realistic setting with time-varying objectives. We consider a real-world problem of tuning a Model Predictive Control (MPC) controller, which regulates indoor temperature in a commercial building testbed \cite{yang2020implementation}. The simulated building has a surface area of $8500\,\mathrm{m}^2$ and serves approximately 1350 occupants, with heating provided by a district heating network. To control the temperature, the MPC controller regulates the valve position of a radiator and is parameterized by three variables. The goal is to choose the optimal parameter that minimizes a weighted sum of energy consumption, $\mathrm{CO}_2$ emissions, and thermal discomfort. This economic objective is commonly found in building control, which captures the trade-off between energy efficiency and thermal comfort \cite{chen2015model}. To create realistic conditions, we introduce a time-varying heating price following a sinusoidal pattern, resulting in a sequence of time-varying objectives (see detailed formulation in Appendix~\ref{appendix_boptest}). Controller parameters are updated daily, and temperature and heating trajectories are continuously measured throughout each day. Experiments are conducted over 200 iterations with varying initial conditions.

The first row of Fig.~\ref{fig:boptest} presents the cost during the active learning stage. After 200 iterations, vvBO achieves the lowest cost among all baseline methods. The impact of the time-varying heating price is reflected in the left plot. In the early phase, vvBO behaves similarly to baseline methods, focusing on exploration. As more data are collected, it adapts effectively to the varying prices by reducing the temperature when prices are high to save energy (peaks), and increasing the temperature when prices are low to improve comfort (valleys). Consequently, after around 150 iterations, vvBO significantly reduces the peak cost. Details on how the heating price affects the total cost can be found in Appendix \ref{appendix_boptest}. This reveals the ability of vvBO to respond to time-varying objectives through rich measurements.

In the second row of Fig.~\ref{fig:boptest}, the validation of the learned model is illustrated. In the left plot, we vary the heating price and evaluate the economic cost associated with the recommended parameter from each method. vvBO achieves the lowest cost across all price levels, while other baseline methods incur high costs. In the right plot, we consider a new objective that is commonly found in building control \cite{GELAZANSKAS201422,kosek2013overview}, minimizing the peak heating power at specific times of the day. This objective is never encountered during the active learning stage, which evaluates the ability to transfer information across different objectives. vvBO continues to recommend effective parameters across all time instances, whereas CTBO, which achieves comparable performance for the economic objective, performs poorly due to large uncertainty in unseen contexts. Indeed, it requires substantially more data to learn the new context. These results highlight that, by leveraging structured measurements, vvBO can generalize across objectives without modeling changing components, allowing fast accommodation in time-varying and previously unseen settings.

\begin{figure}[t]
    \centering
    \includegraphics[width=\linewidth]{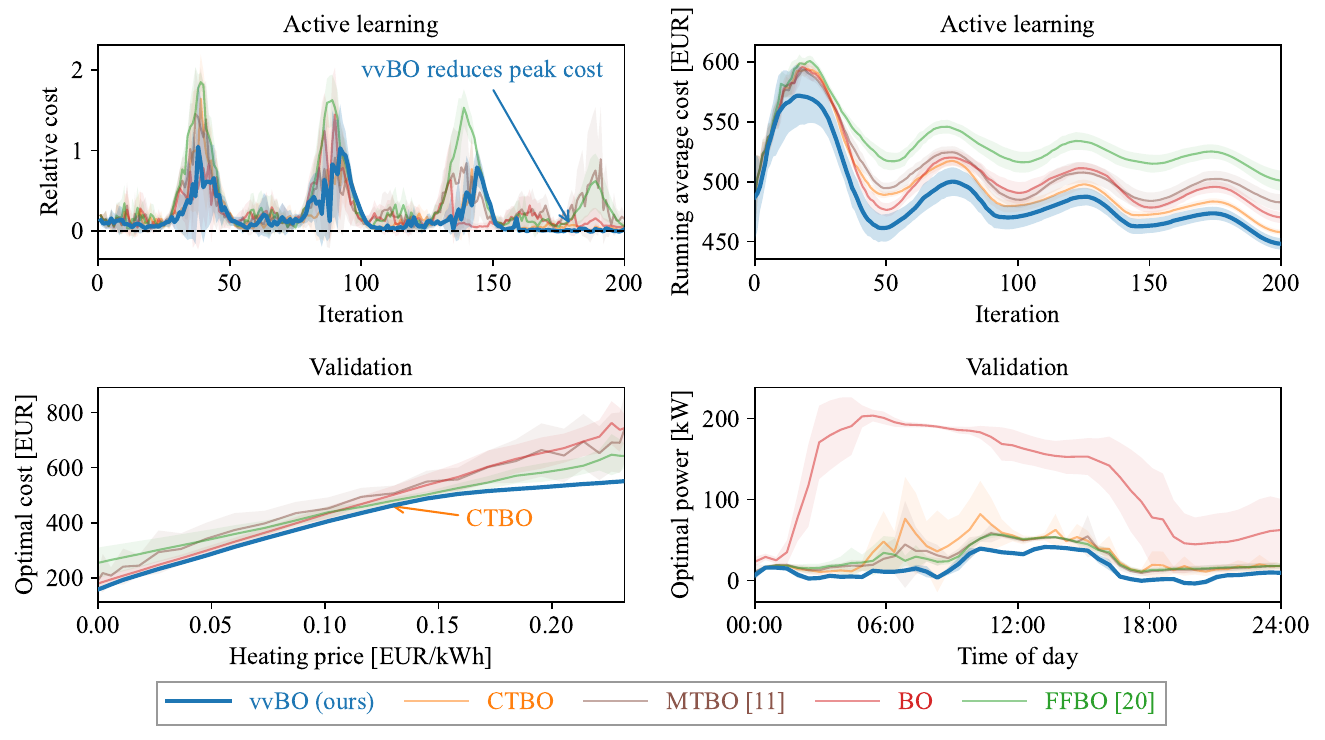}
    \caption{Real-world tuning of an MPC controller for building climate control.}
    \label{fig:boptest}
\end{figure}
\section{Conclusion and limitations}
\label{conclusion}
In this work, we study Bayesian optimization for systems that produce structured, vector-valued outputs observed through linear measurement operators. By modeling directly in the measurement space, we derive general concentration bounds for vector-valued KRR in the measurement space. Building on these results, we propose a UCB-based algorithm and establish sublinear regret guarantees for common kernels. This shows that effective learning is also possible when measurements lie in a general Hilbert space. Furthermore, our work highlights the importance of the observation model in Bayesian optimization: when structured measurements are available, each measurement can reveal substantially richer information about the underlying system than scalar observations. Leveraging this structure can facilitate more efficient learning and information sharing across multiple objectives. Numerical experiments validate these benefits in both synthetic and real-world scenarios.

While the proposed framework provides a general approach to Bayesian optimization with structured measurements, it has several limitations. First, similar to standard BO methods, performance may degrade in high-dimensional input spaces due to the curse of dimensionality. Second, our analysis focuses on unconstrained settings, whereas many applications impose constraints on trajectories or structured outputs \cite{wang2024decentralized,wang2025personalized,chen2015model}. Extending the framework to incorporate constraints and safety considerations is an important direction for future work.

Regarding social impact, the proposed framework has the potential to improve decision-making in applications involving complex dynamical systems, such as energy management, climate control, and environmental monitoring, by enabling more efficient use of data and better utilization of structured information. However, as with many data-driven optimization methods, care must be taken to ensure robustness, privacy, fairness, and reliability when deployed in real-world systems.

{\small
\bibliographystyle{plain}
\bibliography{reference}
}

\appendix
\section{Common linear functionals}
\label{appendix_linear_functionals}

In this section, we provide examples of linear functionals in common Hilbert spaces. In all cases, continuous linear functionals admit a representation via the inner product, as characterized by the Riesz representation theorem~\cite{akhiezer1981theory}.

\paragraph{Finite-dimensional spaces.}

In $\mathbb{R}^n$ with inner product
\[
\langle x, y \rangle = \sum_{i=1}^n x_i y_i,
\]
every linear functional $\varphi$ can be written as
\[
\varphi(x) = \sum_{i=1}^n a_i x_i = \langle x, y \rangle,
\]
where $y = (a_1, \dots, a_n) \in \mathbb{R}^n$.

\paragraph{Sequence space.}

The Hilbert space $\ell^2$ consists of square-summable sequences with inner product
\[
\langle x, y \rangle = \sum_{k=1}^\infty x_k y_k.
\]
Every continuous linear functional $\varphi$ admits the representation
\[
\varphi(x) = \sum_{k=1}^\infty c_k x_k = \langle x, y \rangle,
\]
for some $y = \{c_k\}_{k=1}^\infty \in \ell^2$.

\paragraph{Square-integrable function spaces.}

Let $(\Omega,\mu)$ be a measure space. The space $L^2(\Omega,\mu)$ has inner product
\[
\langle f, g \rangle = \int_\Omega f(x)g(x)\, d\mu(x).
\]
By the Riesz representation theorem, every continuous linear functional $\varphi$ can be expressed as
\[
\varphi(f) = \int_\Omega f(x) g(x)\, d\mu(x) = \langle f, g \rangle,
\]
for a unique $g \in L^2(\Omega,\mu)$.

\paragraph{Sobolev spaces.}

Sobolev spaces are Hilbert spaces equipped with inner products involving derivatives. For example, in $H^1(\Omega)$,
\[
\langle u, v \rangle = \int_\Omega \nabla u \cdot \nabla v + uv \, dx.
\]
Every continuous linear functional $\varphi$ admits a representation
\[
\varphi(u) = \langle u, h \rangle_{H^1},
\]
for a unique $h \in H^1(\Omega)$.

\paragraph{Reproducing kernel Hilbert spaces.}

Let $\mathcal{H}$ be a scalar-valued RKHS with kernel $G : \mathcal{X} \times \mathcal{X} \to \mathbb{R}$. By the reproducing property,
\[
f(x) = \langle f, G(x,\cdot) \rangle_{\mathcal{H}}.
\]

The evaluation functional $L(f) = f(x_0)$ is represented by $G(x_0,\cdot)$. More generally, for
\[
L(f) = \sum_{i=1}^n a_i f(x_i),
\]
the representer is
\[
y = \sum_{i=1}^n a_i G(x_i,\cdot).
\]

Integral functionals can also be expressed in this form. For example, let
\[
L(f) = \int_{\mathcal{X}} g(x) f(x)\, dx,
\]
where the integral is well-defined, and the resulting functional is continuous. Using the reproducing property,
\[
L(f) = \int_{\mathcal{X}} g(x) \langle f, G(x,\cdot) \rangle_{\mathcal{H}} dx
= \left\langle f, \int_{\mathcal{X}} g(x) G(x,\cdot)\, dx \right\rangle_{\mathcal{H}}.
\]
Thus, the representer is
\[
y = \int_{\mathcal{X}} g(x) G(x,\cdot)\, dx \in \mathcal{H}.
\]

For two such functionals with weights $g_1$ and $g_2$, their inner product is
\[
\langle y_1, y_2 \rangle_{\mathcal{H}} 
= \int_{\mathcal{X}} \int_{\mathcal{X}} g_1(x) g_2(y) G(x,y)\, dx\, dy.
\]

More generally, for
\[
L(f) = \sum_{i=1}^n a_i \int_{\mathcal{X}} g_i(x) f(x)\, dx,
\]
the corresponding representer is
\[
y = \sum_{i=1}^n a_i \int_{\mathcal{X}} g_i(x) G(x,\cdot)\, dx.
\]

\section{Analytical formulation for the vector-valued estimator}
\label{appendix_lemma_representer_proof}
In the following lemma, we provide the explicit formulation for \eqref{eq_KRR}.
\begin{lemma}
\label{lemma_representer}
    Let Assumptions \ref{assumption_noise}-\ref{assumption_regularity} hold. For any $x\in\mathcal{X}$, $\mu_t(x)$ is given by
    \[
    \mu_t(x) = \sum_{i=1}^tK(x,x_i)M^*\alpha_i,
    \]
    where $\alpha_i\in\mathcal{M}$ solves a linear system of equations, which is given by
    \[
\sum_{j=1}^t(MK(x_i,x_j)M^* + \lambda I_{i = j})\alpha_j = y_i,
\]
where $I_{i = j}$ is the indicator operator, i.e., $I_{i = j} = I$ if $i = j$ and $I_{i = j} = 0$ otherwise.
\end{lemma}
\begin{proof}
    For any $f\in\mathcal{H}$, we can express it as $f = \mu_t + g$, where $g\in\mathcal{H}$ represents the residual. Note that
\begin{align*}
    \sum_{i=1}^t \|Mf(x_i) - &y_i\|_{\mathcal{M}}^2 
+ \lambda \|f\|_{\mathcal{H}}^2\\
&= \sum_{i=1}^t \|M\mu_t(x_i) + Mg(x_i) - y_i\|_{\mathcal{M}}^2 
+ \lambda \|\mu_t + g\|_{\mathcal{H}}^2\\
& = \sum_{i=1}^t \|Mg(x_i)\|_{\mathcal{M}}^2 + \sum_{i=1}^t \|M\mu_t(x_i) - y_i\|_{\mathcal{M}}^2  - 2 \sum_{i=1}^t \langle Mg(x_i),y_i - M\mu_t(x_i)\rangle_{\mathcal{M}}\\
&\quad + \lambda \|\mu_t\|_{\mathcal{H}}^2+\lambda \| g\|_{\mathcal{H}}^2 + 2\lambda \langle\mu_t,g\rangle_{\mathcal{H}}\\
 &= \sum_{i=1}^t \|M\mu_t(x_i) - y_i\|_{\mathcal{M}}^2 + \lambda \|\mu_t\|_{\mathcal{H}}^2\\
 &\quad + \sum_{i=1}^t \|Mg(x_i)\|_{\mathcal{M}}^2 + \lambda \| g\|_{\mathcal{H}}^2\\
 &\quad + 2\lambda \langle\mu_t,g\rangle_{\mathcal{H}} - 2 \sum_{i=1}^t \langle Mg(x_i),y_i - M\mu_t(x_i)\rangle_{\mathcal{M}}.
\end{align*}
For the last two term, we have
\begin{align*}
    \lambda\langle\mu_t,g\rangle_{\mathcal{H}}&= \lambda\Big\langle \sum_{i = 1}^t \Phi(x_i)M^*\alpha_i, g\Big\rangle_{\mathcal{H}} \\
    &= \lambda\sum_{i = 1}^t\langle\Phi(x_i)M^*\alpha_i, g\rangle_{\mathcal{H}}\\
    &= \lambda\sum_{i = 1}^t\langle M^*\alpha_i, g(x_i)\rangle_{\mathcal{Y}}\\
    &= \lambda\sum_{i = 1}^t\langle \alpha_i,Mg(x_i)\rangle_{\mathcal{M}},
\end{align*}
\begin{align*}
    \sum_{i=1}^t \langle Mg(x_i),y_i - M\mu_t(x_i)\rangle_{\mathcal{M}} &= \sum_{i=1}^t \Big\langle Mg(x_i),\sum_{j=1}^t(MK(x_i,x_j)M^* + \lambda I_{i = j})\alpha_j - M\mu_t(x_i)\Big\rangle_{\mathcal{M}}\\
    & = \lambda\sum_{i=1}^t \Big\langle Mg(x_i),\sum_{j=1}^t I_{i = j}\alpha_j\Big\rangle_{\mathcal{M}}\\
    & = \lambda\sum_{i=1}^t \langle Mg(x_i),\alpha_i\rangle_{\mathcal{M}}.
\end{align*}
Hence, the minimum is attained at $g = 0$.
\end{proof}

\section{Proof for Theorem \ref{theorem_partial}}
\label{appendix_theorem_partial_proof}
This proof follows the same structure as \cite{chowdhury2021no}, adapted to a general vector-valued setting. We highlight the necessary modifications when they arise.

We first establish the result for the case $M = I$, where the learner observes the full output $f(x)$. The general result in Theorem~\ref{theorem_partial} then follows by applying the same argument to the composed function $Mf$, which takes values in $\mathcal{M}$ and lies in a vector-valued RKHS with kernel $K^M$. To start, we write $K^I = K$ for simplicity. We employ the feature representation of an operator-valued kernel \cite{micchelli2005learning}, writing
\[
K(x,s) = \Phi^*(x)\Phi(s),
\]
where $\Phi(x) \in \mathcal{L}(\mathcal{Y},\mathcal{W})$ for some Hilbert space $\mathcal{W}$. Then there exists \(w \in \mathcal W\) such that
\[
f(\cdot)=\Phi^*(\cdot)w.
\]

Define the operator $\Phi_{X_{t}} \in \mathcal{L}(\mathcal{W}, \mathcal{Y}^t)$ by
\[
(\Phi_{X_{t}} w)_i = \Phi^*(x_i) w, \quad \forall w \in \mathcal{W},
\]
and denote its adjoint by $\Phi_{X_{t}}^*$. The vector-valued KRR estimator $\mu_{t}(x)$ can be expressed as
\begin{align*}
\mu_{t}(x)
&= \Phi^*(x)\Phi_{X_{t}}^*(\Phi_{X_{t}}\Phi_{X_{t}}^* + \lambda I)^{-1}Y_{t} \\
&= \Phi^*(x)\Phi_{X_{t}}^*(\Phi_{X_{t}}\Phi_{X_{t}}^* + \lambda I)^{-1}(\Phi_{X_{t}} w + \bm{\delta}) \\
&= \Phi^*(x)(\Phi_{X_{t}}^*\Phi_{X_{t}} + \lambda I)^{-1}\Phi_{X_{t}}^*(\Phi_{X_{t}} w + \bm{\delta}) \\
&= \Phi^*(x)(\Phi_{X_{t}}^*\Phi_{X_{t}} + \lambda I)^{-1}(\Phi_{X_{t}}^*\Phi_{X_{t}} w + \Phi_{X_{t}}^*\bm{\delta}) \\
&= f(x) - \Phi^*(x)(\Phi_{X_{t}}^*\Phi_{X_{t}} + \lambda I)^{-1}(\lambda w - \Phi_{X_{t}}^*\bm{\delta}),
\end{align*}
where $\bm{\delta} = (\delta_1,\cdots, \delta_{t})\in\mathcal{Y}^t$.

Consequently,
\begin{align*}
\|\mu_{t}(x) - f(x)\|_{\mathcal{Y}}
&= \left\|\Phi^*(x)(\Phi_{X_{t}}^*\Phi_{X_{t}} + \lambda I)^{-1}(\lambda w - \Phi_{X_{t}}^*\bm{\delta})\right\|_{\mathcal{Y}} \\
&\leq \|(\Phi_{X_{t}}^*\Phi_{X_{t}} + \lambda I)^{-1/2}\Phi(x)\|_{\mathrm{op}} \\
&\quad \cdot \|(\Phi_{X_{t}}^*\Phi_{X_{t}} + \lambda I)^{-1/2}(\lambda w - \Phi_{X_{t}}^*\bm{\delta})\|_{\mathcal{W}} \\
&\leq \|\Phi^*(x)(\Phi_{X_{t}}^*\Phi_{X_{t}} + \lambda I)^{-1}\Phi(x)\|_{\mathrm{op}}^{1/2} \\
&\quad \cdot \Big( \|\Phi_{X_{t}}^*\bm{\delta}\|_{(\Phi_{X_{t}}^*\Phi_{X_{t}} + \lambda I)^{-1}} + \lambda^{1/2}\|w\|_{\mathcal{W}} \Big),
\end{align*}
where we used $\|A\|_{\mathrm{op}}\le\|A^*A\|_{\mathrm{op}}^{1/2}$ and the bound
\[
\|w\|_{(\Phi_{X_{t}}^*\Phi_{X_{t}} + \lambda I)^{-1}} 
\leq \lambda^{-1/2} \|w\|_{\mathcal{W}}.
\]

Next, observe that
\begin{align*}
\lambda \Phi^*(x)(\Phi_{X_{t}}^*\Phi_{X_{t}} + \lambda I)^{-1}\Phi(x)
&= \Phi^*(x)\Phi(x) - \Phi^*(x)\Phi_{X_{t}}^*(\Phi_{X_{t}}\Phi_{X_{t}}^* + \lambda I)^{-1}\Phi_{X_{t}}\Phi(x) \\
&= K(x,x) - K_{xX_{t}} (K_{X_{t}X_{t}} + \lambda I)^{-1} K_{X_{t}x} \\
&= K_{t}(x,x),
\end{align*}
we have
\[
\|\Phi^*(x)(\Phi_{X_{t}}^*\Phi_{X_{t}} + \lambda I)^{-1}\Phi(x)\|_{\mathrm{op}}^{1/2}
= \lambda^{-1/2} \|K_{t}(x,x)\|_{\mathrm{op}}^{1/2}.
\]

For the noise term, applying Lemma~3 of \cite{chowdhury2021no} (with $\mathcal{W} = \ell^2$) yields that, with probability at least $1-\zeta$, for all $t \geq 1$,
\[
\|\Phi_{X_{t}}^*\bm{\delta}\|_{(\Phi_{X_{t}}^*\Phi_{X_{t}} + \lambda I)^{-1}}
\leq 
\sigma \sqrt{2\log(1/\zeta) + \log\det\bigl(I + \lambda^{-1}\Phi_{X_{t}}^*\Phi_{X_{t}}\bigr)}.
\]

Using Sylvester's determinant identity \cite{IsozakiKorotyaev2011},
\[
\det\bigl(I + \lambda^{-1}\Phi_{X_{t}}^*\Phi_{X_{t}}\bigr)
= \det\bigl(I + \lambda^{-1}\Phi_{X_{t}}\Phi_{X_{t}}^*\bigr)
= \det\bigl(I + \lambda^{-1}K_{X_{t}X_{t}}\bigr).
\]

Combining the above bounds and observing that $\|w\|_{\mathcal{W}} = \|f\|_{\mathcal{H}} \leq \Gamma$, we obtain
\[
\|f(x)-\mu_{t}(x)\|_{\mathcal{Y}}
\leq 
\left(
\Gamma + \frac{\sigma}{\sqrt{\lambda}}
\sqrt{2\log(1/\zeta) + \log\det\bigl(I + \lambda^{-1}K_{X_{t}X_{t}}\bigr)}
\right)
\|K_{t}(x,x)\|_{\mathrm{op}}^{1/2}.
\]
Applying the same argument to the composed function \(Mf\), whose RKHS kernel is \(K^M\), yields Theorem~\ref{theorem_partial}.

\section{Proof for Theorem \ref{theorem_regret}}
\label{appendix_proof_theorem_2}
We first present a useful lemma from \cite{chowdhury2021no}.
\begin{lemma}
\label{appendix_lemma_1}
    For any $\lambda>0$ and $t\geq1$, we have 
    \[
    \log\det\bigl(I + \lambda^{-1}K_{X_tX_t}^{M}\bigr) = \sum_{i = 1}^t\log\det\bigl(I + \lambda^{-1}K_{i-1}^{M}(x_i,x_i)\bigr).
    \]
\end{lemma}

From Lemma \ref{appendix_lemma_1}, we know that $\log\det\bigl(I + \lambda^{-1}K_{X_tX_t}^{M}\bigr)$ is an increasing function with $t$. Hence, we conclude that $\beta_t\leq \beta_{t+1}$ for all $t\geq 0$.

We can write the regret as 
\begin{align*}
    \mathcal{R}_T &= \sum_{t = 1}^TF(x^*)-F(x_t)\\
    & = \sum_{t = 1}^TF(x^*)-\overline{F}_{t-1}(x_t) + \overline{F}_{t-1}(x_t) -F(x_t)\\
    & = \sum_{t = 1}^TF(x^*)-\overline{F}_{t-1}(x^*) + \overline{F}_{t-1}(x^*) - \overline{F}_{t-1}(x_t) +\overline{F}_{t-1}(x_t) - F(x_t)\\
    &\leq 2\sum_{t = 1}^T\beta_{t-1}\|m\|_{\mathcal M}\|K_{t-1}^{M}(x_t,x_t)\|_{\mathrm{op}}^{1/2}\\
    &\leq 2\beta_{T+1}\|m\|_{\mathcal M}\sum_{t = 1}^T\|K_{t-1}^{M}(x_t,x_t)\|_{\mathrm{op}}^{1/2}\\
    &\leq 2\beta_{T+1}\|m\|_{\mathcal M}\sqrt{T\sum_{t = 1}^T\|K_{t-1}^{M}(x_t,x_t)\|_{\mathrm{op}}}.
\end{align*}
Substituting the definition of $\beta_{T+1}$ yields the result.

\section{Proof for Theorem \ref{theorem_kernel_specific}}
\label{appendix_theorem_3_proof}
The following lemma shows that the operator-valued information quantities can be reduced to scalar information gain with separable kernels. This reduction is key to obtaining explicit regret bounds.
\begin{lemma}
\label{lemma_gamma}
Let $K = G \otimes B$. The operator-valued quantities in Theorem~\ref{theorem_regret} are upper bounded by the scalar information gain associated with $G$. In particular,
\begin{align*}
\log\det\bigl(I + \lambda^{-1}G_{X_TX_T}\otimes B_M \bigr)
&\leq \sum_{i  = 1}^{\infty} \gamma_T\!\left(G, \frac{\lambda}{\lambda_i^{B_M}}\right), \\
\sum_{t=1}^T \|K_{t-1}^M(x_t,x_t)\|_{\mathrm{op}}
&\leq 2 C_K \, \gamma_T(G,\frac{\lambda}{C_K}),
\end{align*}
where $\{\lambda_i^{B_M}\}_i$ denotes the point spectrum of $B_M$, and $\gamma_T(G,\lambda) = \max_{X_T}\log\det\!\bigl(I + \lambda^{-1}G_{X_TX_T}\bigr)$ is the information gain associated with the scalar kernel $G$.
\end{lemma}

The proof of Lemma~\ref{lemma_gamma} is provided in Appendix \ref{appendix_proof_lemma_3}. For many commonly used scalar kernels $G$, the information gain $\gamma_T(G,\lambda)$ admits sublinear growth in $T$. However, one can observe that $\log\det\bigl(I + \lambda^{-1}G_{X_TX_T}\otimes B_M \bigr)$ is upper bounded by an infinite sum of information gains with increasing noise variance, which is not guaranteed to be finite. To guarantee a sublinear regret bound, we further assume that the spectrum decomposition of $B_M$ is composed of a finite series, i.e., $B_M = \sum_{i = 1}^{n}\lambda^{B_M}_i\phi_i\phi_i^*$. Applying the result from Lemma \ref{lemma_gamma} in Theorem \ref{theorem_regret}, we can bound the cumulative regret by
\begin{align*}
    \mathcal{R}_T\leq 2\|m\|_{\mathcal{M}}\left(\Gamma + \frac{\sigma}{\sqrt{\lambda}}\sqrt{2\log(\frac{1}{\zeta})+ \sum_{i  = 1}^{n} \gamma_T\!\left(G, \frac{\lambda}{\lambda_i^{B_M}}\right)\bigr)}\right)\sqrt{2 C_K T \gamma_T(G,\frac{\lambda}{C_K})}.
\end{align*}
For the linear kernel,
\[
\gamma_T = \mathcal{O}(\log T),
\]
for the square exponential kernel,
\[
\gamma_T = \mathcal{O}((\log T)^{d+1}),
\]
for the Mat\'ern kernel with $\nu>1$,
\[
\gamma_T = \mathcal{O}(T^{d(d+1)/(2\nu+d(d+1))}\log T).
\]
Consequently, we have 
\begin{enumerate}
    \item Linear kernel: $\mathcal{R}_T \leq \mathcal{O}\bigl(\log T\sqrt{T}\bigr)$,
    \item Gaussian kernel: $\mathcal{R}_T \leq \mathcal{O}\bigl((\log T)^{d+1}\sqrt{T})$,
    \item Mat\'ern kernel: $\mathcal{R}_T \leq \mathcal{O}\bigl(T^{d(d+1)/(2\nu+d(d+1))}\log T\sqrt{T}\bigr)$.
\end{enumerate}

\section{Proof for Lemma \ref{lemma_gamma}}
\label{appendix_proof_lemma_3}
For the first part, according to the Fredholm determinant for positive definite trace class operators \cite{minh2017infinite}, we observe that 
\begin{align*}
    \log\det\!\bigl(I + \lambda^{-1}G_{X_TX_T}\otimes B_M\bigr) = \sum_{t = 1}^T\sum_{i = 1}^{\infty} \log(1 + \frac{1}{\lambda}\lambda_t^G\lambda_i^{B_M})\leq \frac{1}{\lambda}\sum_{t = 1}^T\sum_{i = 1}^{\infty}\lambda_t^G\lambda_i^{B_M}<\infty.
\end{align*}
Hence, applying Fubini's theorem for infinite series (switching the order of summation for an absolutely convergent series), we obtain
\begin{align*}
    \log\det\!\bigl(I + \lambda^{-1}G_{X_TX_T}\otimes B_M\bigr) &= \sum_{t = 1}^T\sum_{i = 1}^{\infty} \log(1 + \frac{1}{\lambda}\lambda_t^G\lambda_i^{B_M})\\
    & = \sum_{i = 1}^{\infty} \sum_{t = 1}^T\log(1 + \frac{1}{\lambda}\lambda_t^G\lambda_i^{B_M})\\
    & = \sum_{i = 1}^{\infty}\log\det\!\bigl(I + \frac{\lambda_i^{B_M}}{\lambda}G_{X_TX_T})\\
    & \leq \sum_{i = 1}^{\infty}\gamma_T(G,\frac{\lambda}{\lambda_i^{B_M}}).
\end{align*}

For the second part, we first present a useful lemma.
\begin{lemma}
\label{appendix_lemma_2}
The operator norm of $K_{t-1}^M(x_t,x_t)$ is given by 
    \begin{align*}
    \|K_{t-1}^M(x_t,x_t)\|_{\mathrm{op}} &= \max_{j} \lambda^{B_M}_j\Big(G(x_t,x_t) - G_{x_tX_{t-1}}\Big(G_{X_{t-1}X_{t-1}} + \frac{\lambda}{\lambda^{B_M}_j}\Big)^{-1}G_{X_{t-1}x_t}\Big),
    \end{align*}
where $\lambda_j^{B_M}$ is the point spectrum of $B_M$.
\end{lemma}
\begin{proof}
    Given that $B_M$ is a trace-class operator, we denote its spectrum decomposition by $B_M = \sum_{j=1}^{\infty}\lambda_{j}^{B_M}\phi_j\phi_j^*$. Similarly, $G_{X_{t-1}X_{t-1}} = \sum_{i = 1}^{t-1}\lambda_i^G\theta_i\theta_i^*$.
    Consequently, we have
    \begin{align*}
        (G_{X_{t-1}X_{t-1}}\otimes B_M + \lambda I)^{-1} &= \sum_{i=1}^{t-1}
        \sum_{j=1}^{\infty}
        (\lambda_i^G\lambda_j^{B_M}\!+\!\lambda)^{-1}
        (\theta_i\otimes\phi_j)
        (\theta_i\otimes\phi_j)^*\\
        &=\sum_{i=1}^{t-1}
        \sum_{j=1}^{\infty}
        (\lambda_i^G\lambda_j^{B_M}\!+\!\lambda)^{-1}
        (\theta_i\otimes\phi_j)
        (\theta_i^*\otimes\phi_j^*)\\
        &=\sum_{j=1}^{\infty}\sum_{i=1}^{t-1}
        (\lambda_i^G\lambda_j^{B_M}\!+\!\lambda)^{-1}
        (\theta_i\theta_i^*)\otimes(\phi_j\phi_j^*)\\
        &=\sum_{j=1}^{\infty}\Big(\lambda_j^{B_M}G_{X_{t-1}X_{t-1}} + \lambda I\Big)^{-1}\otimes(\phi_j\phi_j^*).
    \end{align*}
    For $K_{t-1}^M(x_t,x_t)$, we have 
    \begin{align*}
        K_{t-1}^M(&x_t,x_t)\\
        &= G(x_t,x_t)\otimes B_M - (G_{x_tX_{t-1}}\otimes B_M)(G_{X_{t-1}X_{t-1}}\otimes B_M + \lambda I)^{-1}(G_{X_{t-1}x_t}\otimes B_M)\\
        &= \sum_{j=1}^{\infty}G(x_t,x_t)\lambda_{j}^{B_M}\phi_j\phi_j^* - \sum_{j=1}^{\infty}G_{x_{t}X_{t-1}}\Big(\lambda_j^{B_M}G_{X_{t-1}X_{t-1}} + \lambda I\Big)^{-1}G_{X_{t-1}x_t}(B_M\phi_j\phi_j^*B_M)\\
        &= \sum_{j=1}^{\infty}\Big(G(x_t,x_t)\lambda_{j}^{B_M} - (\lambda_{j}^{B_M})^2G_{x_{t-1}X_{t-1}}\Big(\lambda_j^{B_M}G_{X_{t-1}X_{t-1}} + \lambda I\Big)^{-1}G_{X_{t-1}x_t}\Big)\phi_j\phi_j^*.
    \end{align*}
    Following this expression, we conclude the proof by choosing the largest point spectrum, which is a well-defined value given that $K_{t-1}^M(x,x)$ is a bounded positive trace-class operator.
\end{proof}
Now, according to Mercer's theorem \cite{micchelli2006universal}, there exist a mapping $\varphi:\mathcal{X}\rightarrow \ell^2$ such that $G(x,x) = \varphi^{\top}(x)\varphi(x)$. We also define an operator $\varphi_{X_{t-1}}:\ell^2\rightarrow\mathbb{R}^{t-1}$ via $\varphi_{X_{t-1}}a = [\varphi(x_1)^{\top}a, \dots, \varphi(x_{t-1})^{\top}a]^{\top}, \forall a \in \ell^2$. Then $G_{X_{t-1}X_{t-1}}$ can be expressed as $\varphi_{X_{t-1}}\varphi_{X_{t-1}}^{\top}$. Then, applying the Woodbury identity we have
\begin{align*}
    G(x,x) - G_{xX_{t-1}}&(G_{X_{t-1}X_{t-1}} + \lambda I)^{-1}G_{X_{t-1}x}\\
    &=\varphi(x)^{\top}\varphi(x) - \varphi(x)^{\top}\varphi_{X_{t-1}}^{\top}(\varphi_{X_{t-1}}\varphi_{X_{t-1}}^{\top} + \lambda I)^{-1}\varphi_{X_{t-1}}\varphi(x)\\
    &=\lambda\varphi(x)^{\top}(\varphi_{X_{t-1}}^{\top}\varphi_{X_{t-1}} + \lambda I)^{-1}\varphi(x).
\end{align*}
Applying Lemma \ref{appendix_lemma_2}, we have 
\begin{align*}
    \|K_{t-1}^M(x_t,x_t)\|_{\mathrm{op}} &= \max_{j} \lambda^{B_M}_j\Big(G(x_t,x_t) - G_{x_tX_{t-1}}\Big(G_{X_{t-1}X_{t-1}} + \frac{\lambda}{\lambda^{B_M}_j}I\Big)^{-1}G_{X_{t-1}x_t}\Big)\\
    & = \max_{j}\lambda\varphi(x_t)^{\top}(\varphi_{X_{t-1}}^{\top}\varphi_{X_{t-1}} + \frac{\lambda}{\lambda^{B_M}_j} I)^{-1}\varphi(x_t)\\
    & \leq \lambda\varphi(x_t)^{\top}(\varphi_{X_{t-1}}^{\top}\varphi_{X_{t-1}} + \frac{\lambda}{C_K} I)^{-1}\varphi(x_t)\\
    & = C_K\Big(G(x_t,x_t) - G_{x_tX_{t-1}}(G_{X_{t-1}X_{t-1}} + \frac{\lambda}{C_K} I)^{-1}G_{X_{t-1}x_t}\Big).
\end{align*}
Summing over $t$, we have
\begin{align*}
    \sum_{t = 1}^T\|K_{t-1}^M(x_t,x_t)\|_{\mathrm{op}}&\leq C_K\sum_{t = 1}^TG(x_t,x_t) - G_{x_tX_{t-1}}(G_{X_{t-1}X_{t-1}} + \frac{\lambda}{C_K} I)^{-1}G_{X_{t-1}x_t}\\
    &\leq C_K \log \det(I  + \frac{C_K}{\lambda}G_{X_TX_T})\\
    &\leq 2C_K\gamma_T(G, \frac{\lambda}{C_K}).
\end{align*}

\section{Extension to time-varying and nonlinear objectives}
\label{appendix_extension}

\subsection{Time-varying objective}

When the functional $m_t$ varies over time, the objective becomes $F_t(x) = \langle m_t, M f(x)\rangle_{\mathcal{M}}$. At each iteration, the objective admits uncertainty bounds $F_t(x) \in [\underline{F}_t(x), \overline{F}_t(x)]$, where
\begin{align*}
\overline{F}_t(x)
&= \langle m_t, M \mu_{t-1}(x) \rangle_{\mathcal{M}}
+ \|m_t\|_{\mathcal{M}}\, \|M f(x) - M \mu_{t-1}(x)\|_{\mathcal{M}}, \\
\underline{F}_t(x)
&= \langle m_t, M \mu_{t-1}(x) \rangle_{\mathcal{M}}
- \|m_t\|_{\mathcal{M}}\, \|M f(x) - M \mu_{t-1}(x)\|_{\mathcal{M}}.
\end{align*}

Since only the functional $m_t$ varies while the estimator of $M f$ remains unchanged, the uncertainty quantification in the measurement space can be reused across time. This allows the framework to handle time-varying objectives without modifying the underlying estimator.

Based on this formulation, we present Algorithm~\ref{alg:ucb_tv} for sequential optimization under time-varying objectives. The algorithm follows the same principle as its static counterpart, with the key difference that the confidence width $\beta_{t-1}\|m_t\|_{\mathcal{M}}\,\|K_{t-1}^{M}(x,x)\|_{\mathrm{op}}^{1/2}$ now depends on the current objective $m_t$. This adaptation allows the uncertainty quantification to adjust dynamically to changes in the objective, leading to more accurate exploration–exploitation trade-offs.

\begin{algorithm}[t]
\caption{Time-varying vector-valued Bayesian optimization}
\label{alg:ucb_tv}
\begin{algorithmic}[1]
\REQUIRE Total budget $T$, regularization parameter $\lambda > 0$, confidence level $\zeta \in (0,1)$
\STATE Initialize dataset $X_0=\emptyset$, $Y_0=\emptyset$ with $M\mu_0(x) = 0$, $K_0^M(x,x) = K^M(x,x)$
\FOR{$t = 1,\dots,T$}
    \STATE Compute $M\mu_{t-1}(x)$ and $K_{t-1}^M(x,x)$
    \STATE Select
    \[
    x_t \in \arg\max_{x \in \mathcal{X}} 
    \langle m_t, M \mu_{t-1}(x)\rangle_{\mathcal{M}}
    + \beta_{t-1}\|m_t\|_{\mathcal{M}}\|K_{t-1}^{M}(x,x)\|_{\mathrm{op}}^{1/2}
    \]
    \STATE Query $f$ at $x_t$ and observe $y_t = M f(x_t) + \delta_t$
    \STATE Update $X_t, Y_t$
\ENDFOR
\end{algorithmic}
\end{algorithm}

\subsection{Nonlinear Lipschitz objective}
\label{appendix_nonlinear}

Let $F' : \mathcal{M} \to \mathbb{R}$ be $L_{F'}$-Lipschitz, i.e., $|F'(m_1) - F'(m_2)| \leq L_{F'} \|m_1 - m_2\|_{\mathcal{M}}$. Applying this with $m_1 = M f(x)$ and $m_2 = M \mu_t(x)$ gives
\begin{align*}
F'(M f(x)) 
&\leq F'(M \mu_t(x)) + L_{F'} \|M f(x) - M \mu_t(x)\|_{\mathcal{M}}, \\
F'(M f(x)) 
&\geq F'(M \mu_t(x)) - L_{F'} \|M f(x) - M \mu_t(x)\|_{\mathcal{M}}.
\end{align*}

Thus, the same uncertainty bound in the measurement space can be propagated to nonlinear objectives via Lipschitz continuity. The remainder of the argument follows from the linear case.

\section{Efficient computation}
\label{appendix_efficient_computation}

Consider a separable kernel of the form $K^M(x,s) = G(x,s)\,B_M$, where
\[
B_M = \sum_{i=1}^{\infty} \lambda_i^{B_M} \, \phi_i \phi_i^*
\]
is a positive trace-class operator. Although $B_M$ is trace-class, it can have infinitely many nonzero eigenvalues.

\paragraph{Finite-rank approximation.}
To obtain a tractable representation, we approximate $B_M$ by truncating its spectral decomposition to the leading $n$ components:
\[
B_{M,n} = \sum_{i=1}^{n} \lambda_i^{B_M} \, \phi_i \phi_i^*.
\]
This yields a finite-rank approximation that captures the dominant spectral components and enables efficient computation.

Under this approximation, any observation $y_i$ can be represented in the basis $\{\phi_i\}_{i=1}^n$ as
\[
y_i = \sum_{j=1}^n \bar{Y}_{ij} \phi_j,
\]
where $\bar{Y}_t \in \mathbb{R}^{t \times n}$ collects the corresponding coefficients.

\paragraph{Representation of the estimator.}
Using this basis representation, the coefficients of $M\mu_t(x)$ can be computed explicitly. In particular, the coefficient corresponding to $\phi_i$ is given by
\[
\lambda_i^{B_{M,n}}\, G_{xX_t}
\bigl(\lambda_i^{B_{M,n}} G_{X_tX_t} + \lambda I\bigr)^{-1}
\bar{Y}_{:,i},
\]
where $\bar{Y}_{:,i} \in \mathbb{R}^t$ denotes the $i$-th column of $\bar{Y}_t$.

\paragraph{Efficient evaluation.}
The above representation allows efficient computation of key quantities. For example,
\begin{align*}
\log\det\bigl(\lambda^{-1}& G_{X_tX_t} \!\otimes\! B_{M,n}\! + \!I\bigr)\\
&\!= \!\sum_{j=1}^n \log\det\bigl(\lambda_j^{B_{M,n}}\lambda^{-1} G_{X_tX_t} + I\bigr), \\
\langle m, M\mu_t(x)\rangle_{\mathcal{M}} 
&\!= \!\sum_{i=1}^n \bar{m}_i \lambda_i^{B_{M,n}}\, G_{xX_t}
\bigl(\lambda_i^{B_{M,n}} G_{X_tX_t} + \lambda I\bigr)^{-1}
\bar{Y}_{:,i}, \\
K_t^M(x,x) 
&\!=\!\sum_{j=1}^{n}\!\Big(\!
G(x,x)\lambda_{j}^{B_{M,n}} 
\!- \!(\lambda_{j}^{B_{M,n}})^2 G_{xX_t}\!
\bigl(\lambda_j^{B_{M,n}} \!G_{X_tX_t} \!+ \!\lambda I\bigr)^{-1}\!
G_{X_tx}
\!\Big)\phi_j\phi_j^*,
\end{align*}
where $\bar{m}_i$ is the coordinate of $m$ under basis $\{\phi_i\}_{i=1}^n$, i.e., $m=\sum_{i=1}^n \bar m_i\phi_i$.

\paragraph{Optimization.}
To select the next query point, we solve
\[
x_t \in \arg\max_{x \in \mathcal{X}} \
\langle m, M \mu_{t-1}(x)\rangle_{\mathcal{M}}
+ \beta_{t-1}\|m\|_{\mathcal{M}}\|K_{t-1}^{M}(x,x)\|_{\mathrm{op}}^{1/2}.
\]
This objective is generally non-convex in $x$. For low-dimensional problems, one can discretize the input space and perform grid search. For higher-dimensional settings, gradient-based optimization with multiple random initializations can be used.

\section{Numerical experiments}
\label{appendix_simulation_details}
\subsection{Experiment setup}
\label{appendix_setup}

\paragraph{Objective design.}
To enable a fair comparison across methods with different observation models, we construct the objective independently of the measurement operator. Let $f(x)\in\mathcal Y$ denote the underlying vector-valued output. We define the objective as $F(x)=\langle m,f(x)\rangle_{\mathcal Y}$, $m=\sum_{j=1}^q w_j\xi_j$, where $\Xi=[\xi_1,\dots,\xi_q]:\mathbb R^q\to\mathcal Y$ is a collection of linear functionals represented in $\mathcal Y$, and $w\in\mathbb R^q$ are the associated weights. We compactly write $m=\Xi w$. This formulation is convenient for evaluating vvBO under multiple observation models. In the experiments, we consider:
(i) full observation, where $M=I$, $\mathcal M=\mathcal Y$, and the full function $f(x)$ is observed; (ii) finite-dimensional observation, where $M=\Xi^*$, $\mathcal M=\mathbb R^q$, and only the projections $\Xi^*f(x)
=
[\langle \xi_1,f(x)\rangle_{\mathcal{Y}},\dots,\langle \xi_q,f(x)\rangle_{\mathcal{Y}}]^\top\in\mathbb{R}^q$ are available. Under the finite-dimensional observation model, $F(x)
=
\langle m,f(x)\rangle_{\mathcal Y}
=
\langle w,\Xi^*f(x)\rangle_{\mathbb R^q}$, which provides a unified formulation for benchmark methods with different measurement assumptions. In particular, methods operating on scalar, finite-dimensional, or function-valued measurements can all be compared while optimizing an equivalent objective derived from the same black-box operator $f$.

\paragraph{Implementation of baseline methods.} Here are the implementation details for each baseline method.
\begin{itemize}

\item \emph{Scalar-valued BO (BO).}
For vanilla Bayesian optimization, only the scalar objective value $F(x)$ is measured at each query.

\item \emph{Retrained scalar-valued BO (rBO).}
A variant of BO that is re-initialized whenever the objective changes, discarding previously collected data.

\item \emph{Multi-task BO (MTBO).}
Multi-task BO assumes access to vector-valued observations in a finite-dimensional space \cite{chowdhury2021no}. In our experiments, for each $x\in\mathcal X$, MTBO measures
\[
\Xi^*f(x)\in\mathbb R^q,
\]
and constructs a surrogate model mapping from $x\in\mathcal{X}$ to $\Xi^*f(x)\in\mathbb{R}^q$. The objective is then evaluated as
\[
F(x)=\langle w,\Xi^*f(x)\rangle_{\mathbb R^q}.
\]

For vvBO, when $M=\Xi^*$, vvBO and MTBO operate on the same finite-dimensional observations. In contrast, under full observations ($M=I$), vvBO directly models the full output $f(x)$, while MTBO remains restricted to the projected measurements $\Xi^*f(x)\in\mathbb{R}^q$.

\item \emph{Retrained multi-task BO (rMTBO).}
A variant of MTBO that is re-initialized whenever the objective changes, discarding previously collected data. Since the objective changes through the weights $w$, rMTBO updates the scalarized objective accordingly while reconstructing the model from scratch using the newly collected observations.

\item \emph{Contextual BO (CTBO).}
The objective $F(x)=\langle m,f(x)\rangle_{\mathcal M}$ can equivalently be written as a scalar-valued function $\tilde F:\mathcal X\times\mathcal Y\to\mathbb R$ defined by
\[
\tilde F(x,m):=\langle m,f(x)\rangle_{\mathcal Y}.
\]
Thus, $F(x)=\tilde F(x,m)$ where $m$ acts as a context variable when the objective changes.

If $f$ belongs to a vector-valued RKHS with kernel $K$, then $\tilde F$ lies in a scalar-valued RKHS with kernel \cite{carmeli2006vector,micchelli2005learning}
\[
\tilde K\big((x_1,m_1),(x_2,m_2)\big)
=
\langle m_1,K(x_1,x_2)m_2\rangle_{\mathcal Y}.
\]

Given observations
\[
\tilde X_t=\{(x_i,m_i)\}_{i=1}^t,
\qquad
\tilde Y_t=\{\tilde F(x_i,m_i)\}_{i=1}^t,
\]
the standard scalar-valued GP regression can be applied with posterior mean and variance given by
\begin{align*}
\tilde\mu_t(x,m)
&=
\tilde K_{(x,m)\tilde X_t}
\big(
\tilde K_{\tilde X_t\tilde X_t}+\lambda I
\big)^{-1}
\tilde Y_t,
\\
\tilde\sigma_t^2(x,m)
&=
\tilde K((x,m),(x,m))
-
\tilde K_{(x,m)\tilde X_t}
\big(
\tilde K_{\tilde X_t\tilde X_t}+\lambda I
\big)^{-1}
\tilde K_{\tilde X_t(x,m)}.
\end{align*}

Common acquisition functions such as UCB, expected improvement, and Thompson sampling can then be applied \cite{garnett2023bayesian}. In this work, we use the UCB acquisition function.

\item \emph{Function-on-function BO (FFBO).}
FFBO is implemented following \cite{huang2026function}, corresponding to CTBO with a fixed context. The observations are scalar objective $F(x)$.

\end{itemize}
\paragraph{RKHS representation of vector-valued outputs.}
\label{appendix_output_representation}

In our experiments, the system output $f(x)$ is a function or a trajectory. For numerical implementation, we construct a finite-dimensional representation. Specifically, for each input $x$, the function $f(x)$ is evaluated on a dense grid $\{t_j\}_{j=1}^{50}$, yielding samples $y_j(x) = f(x)(t_j)$. We then approximate $f(x)$ using $\widehat{f(x)}$, which is in a scalar-valued RKHS and expressed as $\widehat{f(x)}(t) = \sum_{j=1}^{50} \alpha_{f(x),j}\, k(t, t_j)$, where $k$ is a scalar-valued kernel defined on the output domain. The coefficient $\alpha_{f(x)} = [\alpha_{f(x),1}, \dots, \alpha_{f(x),50}]^\top \in \mathbb{R}^{50}$ are coefficients obtained by solving a scalar-valued KRR problem
\[ 
\min_{\widehat{f(x)} \in \mathcal{H}_k} \sum_{j=1}^{50} \big(\widehat{f(x)}(t_j) - y_j(x)\big)^2 + 0.01 \|\widehat{f(x)}\|_{\mathcal{H}_k}^2. 
\]
Similarly, each linear functional $\xi$ is represented in the same basis as $\xi(\cdot) = \sum_{j=1}^{50} \alpha_{\xi,j}\, k(\cdot, t_j)$, with coefficients $\alpha_{\xi,j} \in \mathbb{R}$. Linear functionals of $f(x)$ are then evaluated as RKHS inner products in this finite-dimensional representation. All experiments use this finite representation of $f(x)$ and the functionals $\xi$, ensuring a consistent and fair comparison. This construction is introduced solely for numerical implementation and does not alter the theoretical results derived in this work.

\paragraph{Hyperparameters.}
\begin{table}[t]
\centering
\caption{Hyperparameters for different test operators.}
\label{tab:hyperparameters}
\begin{tabular}{lccccc}
\toprule
Test Operator & Kernel & $\lambda$ & Noise level & Input length scale(s) & Output length scale(s)\\
\midrule
GP       & RBF & $10^{-2}$ & $10^{-2}$ & $0.1$ & $0.1$ \\
GP (3D) & RBF & $10^{-2}$ & $10^{-2}$& $0.1$ & $0.1$ \\
Ackley & RBF & $10^{-2}$ & $10^{-2}$ & $3$ & $3$ \\
Eggholder        & RBF & $10^{-2}$ & $10^{0}$ & $50$ & $50$ \\
Bukin    & RBF & $10^{-2}$ & $10^{0}$ & $0.6$ & $1$ \\
Holder Table   & RBF & $10^{-2}$ & $10^{0}$ & $1$ & $1$ \\
Shubert   & RBF & $10^{-2}$ & $10^{-3}$ & $0.5$ & $0.5$ \\
Langermann   & RBF & $10^{-2}$ & $10^{-3}$ & $0.5$ & $0.5$ \\
Controller tuning   & RBF & $10^{-2}$ & $10^{-1}$ & $[0.6,0.8,0.8]$ & $[5000,10000]$ \\
\bottomrule
\end{tabular}
\end{table}

\begin{table}[t]
\centering
\caption{Computation time (ms) of vvBO across different iterations (mean $\pm$ std) for different test operators.}
\label{tab:computation_time_2d}
\begin{tabular}{lccc}
\toprule
Test Operator & Iteration 30 & Iteration 60 & Iteration 100 \\
\midrule
GP           & $8.16 \pm 0.04$ & $8.59 \pm 0.10$ & $9.22 \pm 0.17$ \\
Ackley       & $10.08 \pm 0.22$ & $10.62 \pm 0.23$ & $11.42 \pm 0.19$ \\
Eggholder    & $8.13 \pm 0.19$ & $8.53 \pm 0.18$ & $9.20 \pm 0.20$ \\
Bukin        & $8.24 \pm 0.15$ & $8.56 \pm 0.18$ & $9.39 \pm 0.18$ \\
Holder Table & $10.13 \pm 0.25$ & $10.60 \pm 0.18$ & $11.39 \pm 0.27$ \\
Shubert      & $7.98 \pm 0.17$ & $8.39 \pm 0.15$ & $9.22 \pm 0.26$ \\
Langermann   & $8.06 \pm 0.16$ & $8.42 \pm 0.09$ & $9.11 \pm 0.19$ \\
\bottomrule
\end{tabular}
\end{table}
All simulations are conducted on HP OMEN 45L Desktop (Intel Core Ultra 9 285K, 24 cores up to 6.5~GHz), 62~GB RAM, running Ubuntu 24.04.3~LTS (Linux 6.14). The hyperparameters for different test operators are summarized in Table \ref{tab:hyperparameters} (identical across different baseline methods). The kernel used for vvBO is composed of an RBF kernel and the identity operator. The computation times for those test operators at different iterations for vvBO are shown in Table \ref{tab:computation_time_2d}.

\subsection{Synthetic problem}
\label{appendix_synthetic}

\subsubsection{Test operators}

We evaluate the proposed method on synthetic problems constructed from standard optimization benchmarks. For each benchmark, the input is decomposed as $u=(x,t)$, where $x$ denotes the input variable and $t$ indexes the structured output. The latent vector-valued function is generated as
\[
f(x)(t) = h(x,t),
\]
where $h$ is one of the following ground-truth functions. The scalar objective is then obtained by applying a linear functional to the trajectory $f(x)$.

\paragraph{Gaussian process.}
For the GP benchmark, we generate a smooth function using an RBF kernel expansion. Let $X_{N_x} = \{x_i\}_{i=1}^{N_x}$ and $T_{N_t} = \{t_j\}_{j=1}^{10}$ denote fixed grids on $[0,1]$. $N_x$ changes with grid density. For one-dimensional inputs, $N_x = 10$. For three-dimensional inputs, $N_x = 125$. Let $\alpha\in \mathbb{R}^{N_x\times N_t}$ be a matrix with randomly sampled elements, i.e., $\alpha_{ij}$ is sampled uniformly from $[-3.5,3.5]$. The ground truth is $h(x,t)
=G_{xX_{N_x}}\alpha G_{T_{N_t}t}$, where $G$ is an RBF kernel with length scale $\ell=0.1$. We implement both a one-dimensional version, denoted GP, and a higher-dimensional version, denoted GP (3D), where $x\in\mathbb{R}^3$.

\paragraph{Ackley.}
For $u\in\mathbb{R}^d$, the Ackley benchmark is
\[
h(u)
=
-\left(
-20\exp\!\left(-0.2\sqrt{\frac{1}{d}\sum_{i=1}^d u_i^2}\right)
-\exp\!\left(\frac{1}{d}\sum_{i=1}^d \cos(0.2\pi u_i)\right)
+20+e-20
\right).
\]
The domain is $u_i\in[-32.768,32.768]$.

\paragraph{Eggholder.}
For $u=(u_1,u_2)$,
\[
h(u)
=
-\left(\frac{u_2}{2}+47\right)
\sin\!\left(
\sqrt{\frac{1}{2}\left|\frac{u_2}{2}+\frac{u_1}{4}+47\right|}
\right)
-\frac{u_1}{2}
\sin\!\left(
\sqrt{\frac{1}{2}\left|\frac{u_1}{2}-\left(\frac{u_2}{2}+47\right)\right|}
\right),
\]
with domain $u_i\in[-512,512]$.

\paragraph{Bukin.}
For $u=(u_1,u_2)$,
\[
h(u)
=
-100\sqrt{\left|u_2-0.01u_1^2\right|}
+0.01|u_1+10|
+180,
\]
with $u_1\in[-15,-5]$ and $u_2\in[-3,3]$.

\paragraph{Holder Table.}
For $u=(u_1,u_2)$,
\[
h(u)
=
\left|
\sin(u_1)\cos(u_2)
\exp\!\left(
\left|1-\frac{\sqrt{u_1^2+u_2^2}}{\pi}\right|
\right)
\right|,
\]
with $u_i\in[-10,10]$.

\paragraph{Shubert.}
For $u=(u_1,u_2)$,
\[
h(u)
=
\frac{1}{100}
\left[
\sum_{i=1}^{5} i\cos\!\left((i+1)\frac{u_1}{2}+i\right)
\right]
\left[
\sum_{i=1}^{5} i\cos\!\left((i+1)\frac{u_2}{2}+i\right)
\right],
\]
with $u_i\in[-10,10]$.

\paragraph{Langermann.}
For $u=(u_1,u_2)$,
\[
h(u)
=
\sum_{i=1}^{5}
c_i
\exp\!\left(
-\frac{\|u/2-A_i\|_2^2}{\pi}
\right)
\cos\!\left(\pi \|u/2-A_i\|_2^2\right),
\]
where
\[
c=(1,2,5,2,3),
\qquad
A=
\begin{pmatrix}
3 & 5\\
5 & 2\\
2 & 1\\
1 & 4\\
7 & 9
\end{pmatrix},
\]
with $u_i\in[0,10]$.

\subsubsection{Implementation details}

We consider three phases with changing objectives. Each objective is defined as $F_i(x)=\langle m_i,f(x)\rangle_{\mathcal Y}$, $m_i=\Xi_i w_i$, where $\Xi_i=[\xi_{i,1},\dots,\xi_{i,5}]:\mathbb R^5\to\mathcal Y$ collects a set of base functionals and $w_i\in\mathbb R^5$ are the associated weights. From phase one to two, only the weights $w_i$ change while the functional basis $\Xi_i$ remains fixed. From phase two to three, both $\Xi_i$ and $w_i$ are modified. For vvBO, we consider two observation regimes. Under full observations ($M=I$), the full output $f(x)$ is observed and all three phases are evaluated. Under partial observations ($M=\Xi^*$), only projected measurements $\Xi^*f(x)\in\mathbb R^5$ are available. In this case, $F_i(x)
=
\langle m_i,f(x)\rangle_{\mathcal Y}
=
\langle w_i,\Xi_i^*f(x)\rangle_{\mathbb R^5}$. Since the measurement operator is fixed to $\Xi^*$, only the transition from phase one to two is considered under partial observations, while phase three is omitted because changing $\Xi_i$ would also modify the measurement operator.

The corresponding functional bases $\Xi_i$ and weights $w_i$ for each test operator are summarized in Table~\ref{tab:measurements}. Here, $\xi_{\mathrm{int}}$ denotes a randomly generated integral functional, while $\xi(t)$ denotes a point-evaluation functional at $t$. Table~\ref{tab:synthetic_measurements} summarizes the inputs and measurements available to each method across phases. In particular, vvBO with full observations ($M=I$) directly models the full output $f(x)$, whereas vvBO with partial observations ($M=\Xi^*$) only accesses the projected measurements $\Xi^*f(x)\in\mathbb{R}^5$. MTBO and rMTBO operate on finite-dimensional observations constructed from $\Xi_i^*f(x)\in\mathbb{R}^5$, while BO and rBO only observe the scalar objective values. CTBO treats $m_i$ as a context variable and models the scalar-valued mapping $(x, m_i)\mapsto \tilde{F}(x, m_i)$. FFBO corresponds to CTBO with fixed context and only observes scalar objective values.

\begin{table}[t]
\centering
\caption{
Functional basis $\Xi_i$ and weights $w_i$ defining the objectives
$F_i(x)=\langle \Xi_i w_i, f(x)\rangle_{\mathcal Y}$ across phases.
From phase one to two, only the weights $w_i$ change, while from phase two to three both $\Xi_i$ and $w_i$ are modified.}
\label{tab:measurements}
\begin{tabular}{llccc}
\toprule
Test Operator & Phase & Functional basis $\Xi_i$ & $w_i \in \mathbb{R}^5$ & $\beta$ \\
\midrule

\multirow{3}{*}{GP, GP(3D)}
& one & $[\xi(0),\xi(0.1),\xi(0.2),\xi(0.3),\xi(0.4)]$ & $\frac{1}{5}[1,1,1,1,1]$ & 6 \\
& two & same as phase one & $[0,1,0,0,0]$ & 6 \\
& three & $[\xi(0.5),\xi(0.6),\xi(0.7),\xi(0.8),\xi(0.9)]$ & $\frac{1}{5}[1,1,1,1,1]$ & 6 \\

\midrule

\multirow{3}{*}{Ackley}
& one & $[\xi_{\mathrm{int}}^{(1)},\dots,\xi_{\mathrm{int}}^{(5)}]$ & $\frac{1}{5}[1,1,1,1,1]$ & 10 \\
& two & same as phase one & $\frac{1}{4}[1,0,1,1,1]$ & 40 \\
& three & $[\xi_{\mathrm{int}}^{(1)},\dots,\xi_{\mathrm{int}}^{(5)}]$ & $\frac{1}{5}[1,1,1,1,1]$ & 70 \\

\midrule

\multirow{3}{*}{Bukin}
& one & $[\xi(0),\xi(0.5),\xi(1),\xi(1.5),\xi(2)]$ & $\frac{1}{5}[1,1,1,1,1]$ & 100 \\
& two & same as phase one & $[0,0,0,0,1]$ & 115 \\
& three & $[\xi(0),\xi(-0.5),\xi(-1),\xi(-1.5),\xi(-2)]$ & $\frac{1}{5}[1,1,1,1,1]$ & 80 \\

\midrule

\multirow{3}{*}{Eggholder}
& one & $[\xi(500),\xi(400),\xi(300),\xi(200),\xi(100)]$ & $[1,0,0,0,0]$ & 400 \\
& two & same as phase one & $[0,0,1,0,0]$ & 250 \\
& three & $[\xi(0),\xi(-100),\xi(-200),\xi(-300),\xi(-400)]$ & $[0,0,0,0,1]$ & 300 \\

\midrule

\multirow{3}{*}{Holder Table}
& one & $[\xi_{\mathrm{int}}^{(1)},\dots,\xi_{\mathrm{int}}^{(5)}]$ & $\frac{1}{5}[1,1,1,1,1]$ & 30 \\
& two & same as phase one & $[1,0,0,0,0]$ & 30 \\
& three & $[\xi_{\mathrm{int}}^{(1)},\dots,\xi_{\mathrm{int}}^{(5)}]$ & $\frac{1}{5}[1,1,1,1,1]$ & 5 \\

\midrule

\multirow{3}{*}{Shubert}
& one & $[\xi(0),\xi(1),\xi(2),\xi(3),\xi(4)]$ & $\frac{1}{5}[1,1,1,1,1]$ & 0.5 \\
& two & same as phase one & $[0,0,0,1,0]$ & 0.5 \\
& three & $[\xi(0),\xi(-1),\xi(-2),\xi(-3),\xi(-4)]$ & $[0,0,0,0,1]$ & 1 \\

\midrule

\multirow{3}{*}{Langermann}
& one & $[\xi(5),\xi(6),\xi(7),\xi(8),\xi(9)]$ & $\frac{1}{5}[1,1,1,1,1]$ & 3 \\
& two & same as phase one & $[1,0,0,0,0]$ & 3 \\
& three & $[\xi(0),\xi(1),\xi(2),\xi(3),\xi(4)]$ & $[1,0,0,0,0]$ & 3 \\

\bottomrule
\end{tabular}
\end{table}

\begin{table*}[t]
\centering
\caption{
Inputs, measurements, and objective evaluation for different baseline methods in the synthetic experiments.
}
\label{tab:synthetic_measurements}
\small
\begin{tabular}{lllll}
\toprule
Method & Phase & Input & Measurement & Surrogate objective \\
\midrule

\multirow{3}{*}{vvBO ($M=I$)}
& one & $x$ & $f(x)$ 
& $\langle \Xi_1 w_1,f(x)\rangle_{\mathcal Y}$ \\

& two & $x$ & $f(x)$ 
& $\langle \Xi_2 w_2,f(x)\rangle_{\mathcal Y}$ \\

& three & $x$ & $f(x)$ 
& $\langle \Xi_3 w_3,f(x)\rangle_{\mathcal Y}$ \\

\midrule

\multirow{3}{*}{vvBO ($M=\Xi^*$)}
& one & $x$ & $\Xi_1^*f(x)$ 
& $\langle w_1,\Xi_1^*f(x)\rangle_{\mathbb R^5}$ \\

& two & $x$ & $\Xi_2^*f(x)$ 
& $\langle w_2,\Xi_2^*f(x)\rangle_{\mathbb R^5}$ \\

& three & -- & -- & -- \\

\midrule

\multirow{3}{*}{MTBO}
& one & $x$ & $\Xi_1^*f(x)$ 
& $\langle w_1,\Xi_1^*f(x)\rangle_{\mathbb R^5}$ \\

& two & $x$ & $\Xi_1^*f(x)$ 
& $\langle w_1,\Xi_1^*f(x)\rangle_{\mathbb R^5}$ \\

& three & $x$ & $\Xi_1^*f(x)$ 
& $\langle w_1,\Xi_1^*f(x)\rangle_{\mathbb R^5}$ \\

\midrule

\multirow{3}{*}{rMTBO}
& one & $x$ & $\Xi_1^*f(x)$ 
& $\langle w_1,\Xi_1^*f(x)\rangle_{\mathbb R^5}$ \\

& two & $x$ & $\Xi_2^*f(x)$ (re-initialized) 
& $\langle w_2,\Xi_2^*f(x)\rangle_{\mathbb R^5}$ \\

& three & $x$ & $\Xi_3^*f(x)$ (re-initialized) 
& $\langle w_3,\Xi_3^*f(x)\rangle_{\mathbb R^5}$ \\

\midrule

\multirow{3}{*}{BO}
& one & $x$ & $F_1(x)$ & $F_1(x)$ \\
& two & $x$ & $F_2(x)$ & $F_2(x)$ \\
& three & $x$ & $F_3(x)$ & $F_3(x)$ \\

\midrule

\multirow{3}{*}{rBO}
& one & $x$ & $F_1(x)$ & $F_1(x)$ \\
& two & $x$ & $F_2(x)$ (re-initialized) & $F_2(x)$ \\
& three & $x$ & $F_3(x)$ (re-initialized) & $F_3(x)$ \\

\midrule

\multirow{3}{*}{CTBO}
& one & $(x,\Xi_1 w_1)$ & $F_1(x)$ 
& $\tilde F(x,\Xi_1 w_1)$ \\

& two & $(x,\Xi_2 w_2)$ & $F_2(x)$ 
& $\tilde F(x,\Xi_2 w_2)$ \\

& three & $(x,\Xi_3 w_3)$ & $F_3(x)$ 
& $\tilde F(x,\Xi_3 w_3)$ \\

\midrule

\multirow{3}{*}{FFBO}
& one & $(x,\Xi_1 w_1)$ & $F_1(x)$ 
& $\tilde F(x,\Xi_1 w_1)$ \\

& two & $(x,\Xi_1 w_1)$ & $F_2(x)$ 
& $\tilde F(x,\Xi_1 w_1)$ \\

& three & $(x,\Xi_1 w_1)$ & $F_3(x)$ 
& $\tilde F(x,\Xi_1 w_1)$ \\

\bottomrule
\end{tabular}
\end{table*}

\subsubsection{Additional experiments}

\paragraph{Full observation.}
We report additional experiments under full observation ($M = I$), where the complete output $f(x)$ is available. The results are shown in Fig.~\ref{fig_appendix_other_full}. These results confirm that vvBO adjusts effectively across varying objectives and that the benefit of structured measurements persists across different benchmarks.

\paragraph{Partial observation.}
We further evaluate the proposed method under partial observation settings, where the measurement operator maps the output to a finite-dimensional space ($M = \Xi^*$). The results are shown in Fig.~\ref{fig:regret_partial} and Fig.~\ref{fig_appendix_other_partial}. Across all test operators, vvBO consistently outperforms the baseline methods in terms of both simple regret and cumulative regret, effectively accommodating to different objectives. In contrast, scalar-valued BO methods suffer from information loss due to observing only the objective value, while MTBO has limited performance and does not adapt to changes in the objective function. These results demonstrate that vvBO retains its advantage under different observation models.

\begin{figure}[t]
    \centering
    \includegraphics[width=\linewidth]{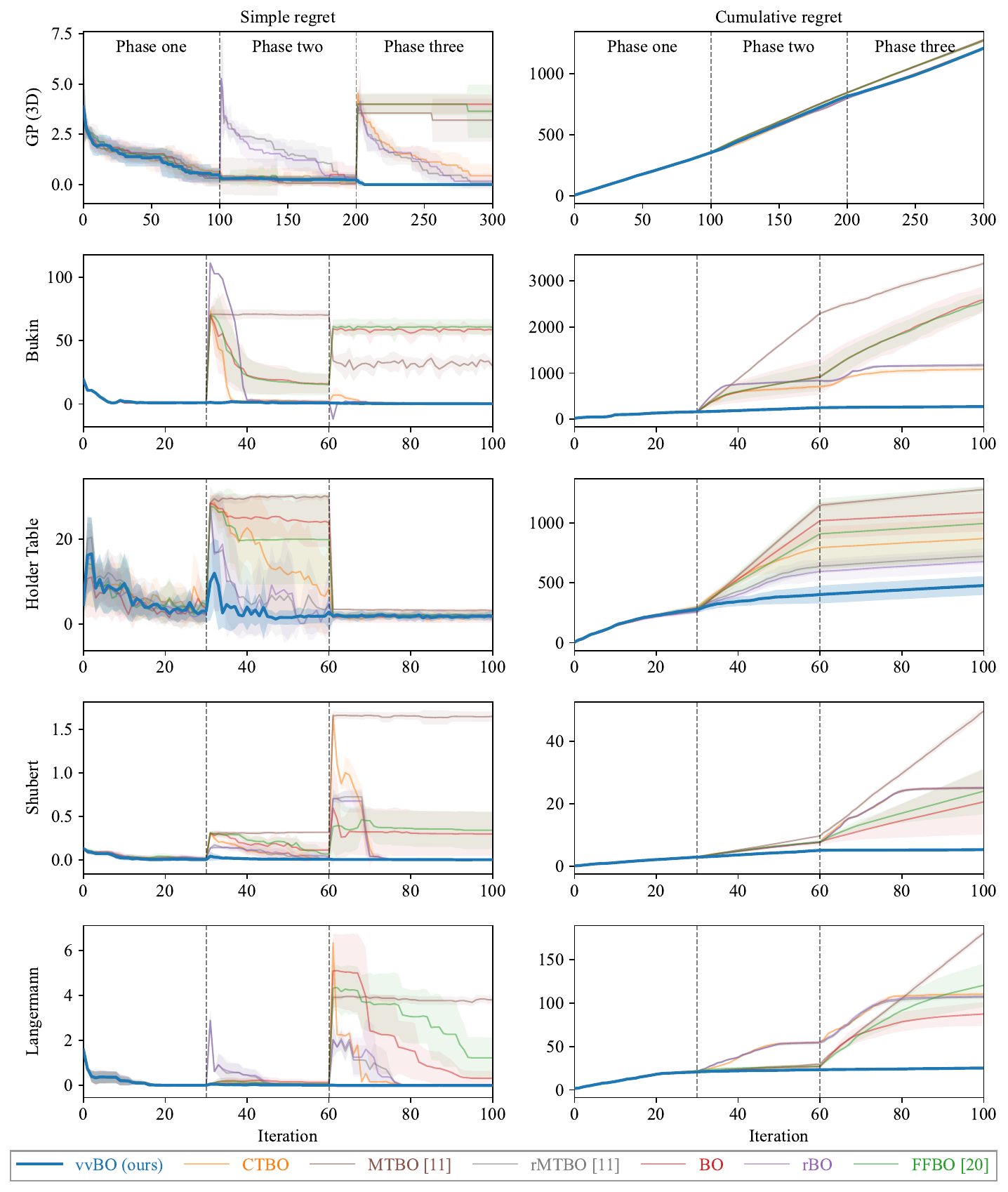}
    \caption{Comparison of simple regret and cumulative regret for additional test operators under full observations ($M = I$).}
    \label{fig_appendix_other_full}
\end{figure}

\begin{figure}[t]
    \centering
    \includegraphics[width=\linewidth]{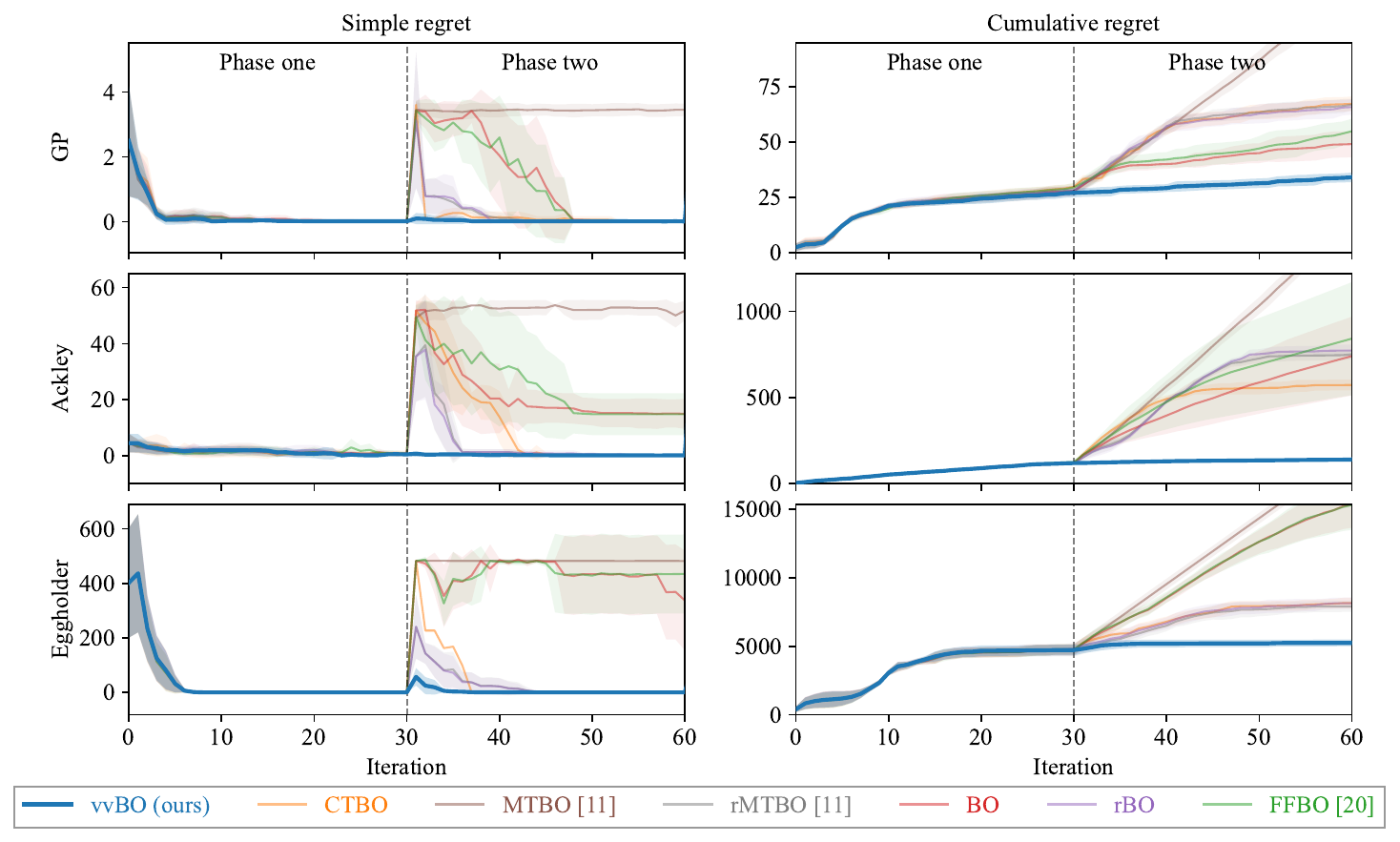}
    \caption{Comparison of simple regret and cumulative regret for three test operators under partial observations ($M = \Xi^*$).}
    \label{fig:regret_partial}
\end{figure}

\begin{figure}[t]
    \centering
    \includegraphics[width=\linewidth]{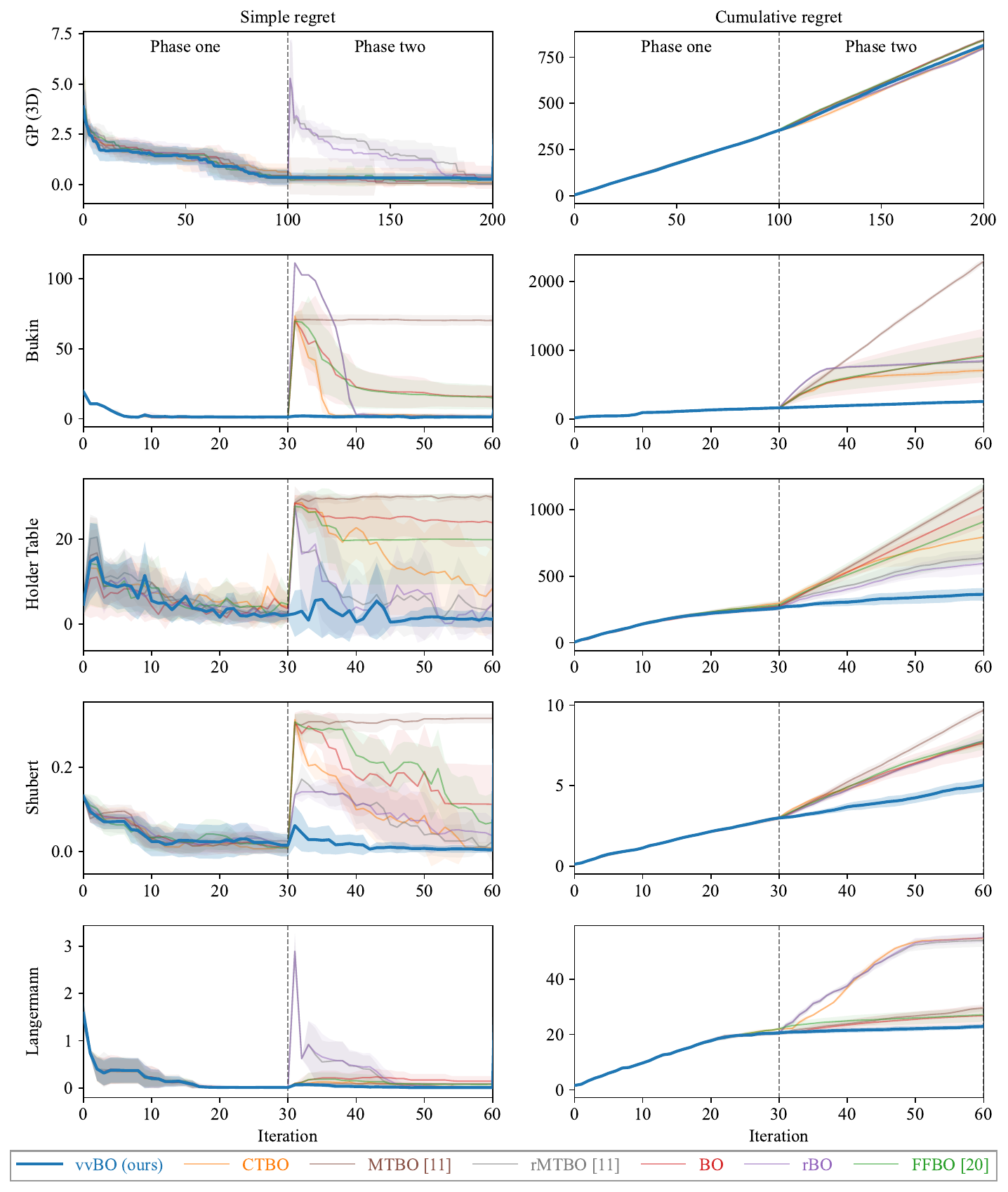}
    \caption{Comparison of simple regret and cumulative regret for additional test operators under partial observations ($M = \Xi^*$).}
    \label{fig_appendix_other_partial}
\end{figure}

\paragraph{Comparison between vvBO and CTBO.}

The performance of CTBO in tracking changing objectives depends critically on the variation of the context $m$. When $m$ shifts to a distant region, CTBO lacks sufficient samples in that region to construct an accurate model, resulting in more exploration. In contrast, when $m$ changes only slightly, CTBO can effectively reuse existing samples from the previous context to reduce uncertainty. This behavior is illustrated in Fig.~\ref{fig_appendix_regret}. From phase one to phase two, we first change the weight from $[1,0,0,0,0]$ to $[0,0,1,0,0]$, corresponding to optimizing only the third component $\xi_3$. This induces a large shift in the context, leading to a significant increase in cumulative regret due to the need for extensive exploration in a previously unseen region. In contrast, when the weight is changed from $[1,0,0,0,0]$ to $[0.8,0.2,0,0,0]$, the induced change is small. In this case, CTBO can leverage previously collected data more effectively, resulting in reduced exploration cost and lower cumulative regret. The corresponding uncertainty estimates for CTBO are shown in Fig.~\ref{fig_appendix_bound}. When transitioning to $[0,0,1,0,0]$, the model exhibits high uncertainty due to the lack of nearby samples. For the smaller change, the uncertainty remains well-controlled, enabling more efficient exploration. These results highlight the key limitation of CTBO. The performance is highly sensitive to the variation of the context. In contrast, vvBO models the underlying vector-valued operator directly and therefore maintains consistent performance regardless of changes in the objective functional.

\begin{figure}[t]
    \centering
    \includegraphics[width=\linewidth]{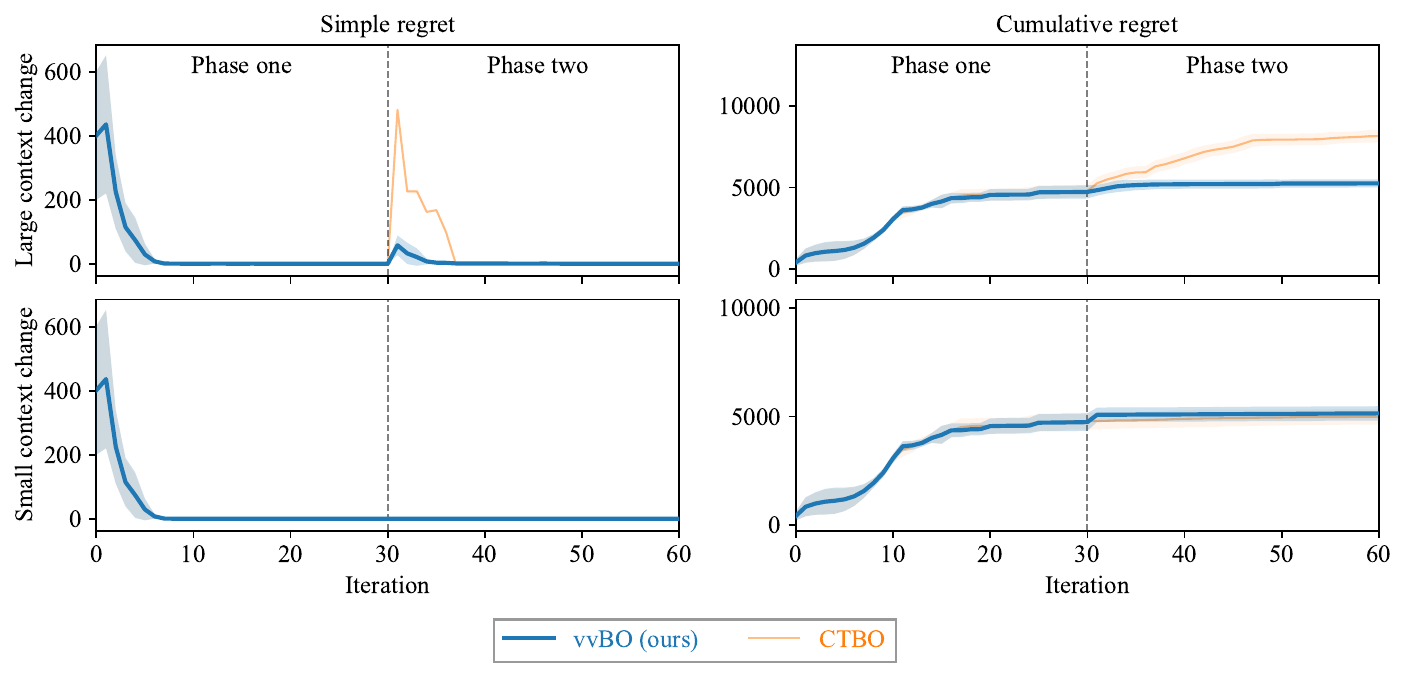}
    \caption{Comparison between vvBO and CTBO for the eggholder test operator for different contexts. When the change in the context is small, CTBO has a lower exploration cost.}
    \label{fig_appendix_regret}
\end{figure}

\begin{figure}[t]
    \centering
    \includegraphics[width=\linewidth]{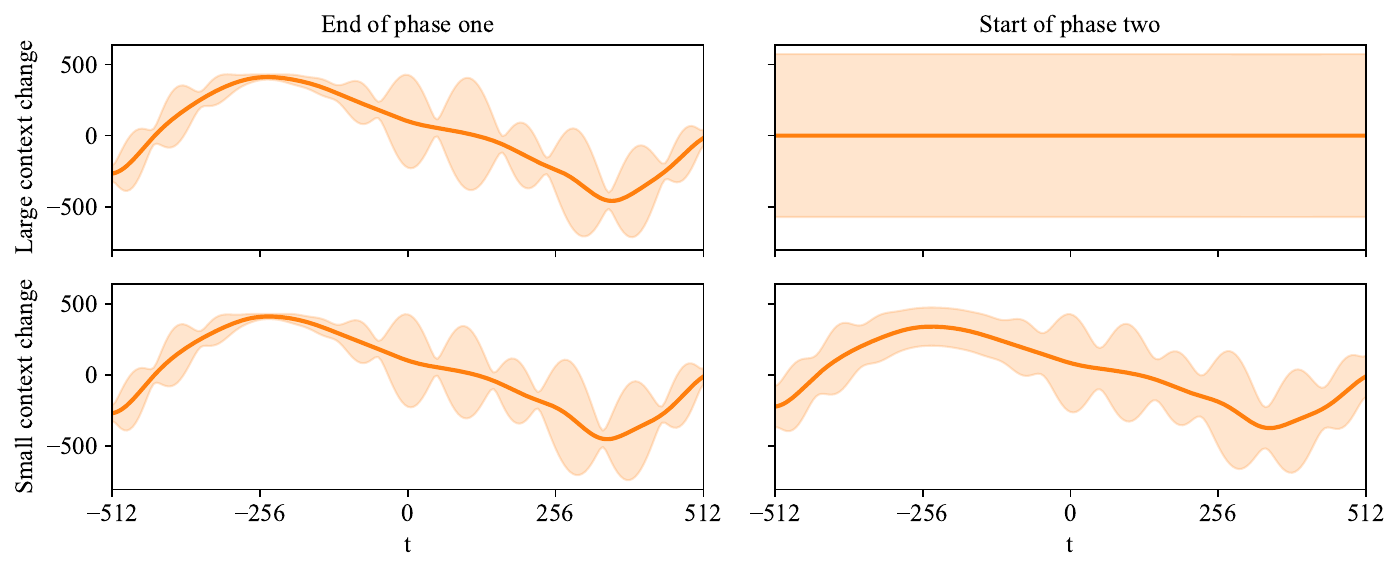}
    \caption{Comparison of the confidence bounds of CTBO for the eggholder test operator for different objectives. When the change in the context is small, CTBO can formulate a more informative confidence bound.}
    \label{fig_appendix_bound}
\end{figure}

\subsection{Real-world controller tuning}
\label{appendix_boptest}

We evaluate the proposed method on a real-world building control task using the Building Optimization Testing Framework (BOPTEST) \cite{blum2021building}, a high-fidelity simulation framework for benchmarking advanced HVAC control strategies. BOPTEST provides physics-based building emulators implemented in Modelica and exposed through standardized APIs, enabling realistic closed-loop evaluation under dynamic weather, occupancy, and pricing conditions \cite{blum2021building}. 

We consider the `\texttt{singlezone\_commercial\_hydronic}' test case, which models a commercial building zone served by a hydronic radiator connected to a district heating network. The building dynamics include thermal interactions between the indoor air, building envelope, external weather conditions, solar gains, and internal heat loads. Heating is delivered through a water-based radiator system, resulting in slow thermal dynamics and delayed control responses that make long-horizon optimization challenging. The controller regulates the radiator valve opening $u_1\in[0,1]$, which determines the heat supplied to the zone through the district heating loop. Further details on this test case can be found in \cite{yang2020implementation}.

\paragraph{Controller design.}
We employ a model predictive control (MPC) scheme to regulate the indoor temperature. The system dynamics are modeled using an AutoRegressive with eXogenous input (ARX) model:
\[
y_{k+1} = \bm{a}^\top \bm{y}_{k-10:k} + \sum_{i=1}^{3} \bm{b}_i^\top \bm{u}_{i,k-10:k},
\]
where $y$ is the indoor temperature, and $u_1, u_2, u_3$ correspond to the valve position, outdoor temperature, and solar irradiation. In this formulation, only the valve position is controlled. The outdoor temperature and solar irradiation are read directly from the BOPTEST API. The system model coefficients are obtained by solving a least squares problem with the data collected from a randomized bang-bang controller, i.e., the radiator is turned on when the temperature falls below a lower bound, and turned off when it exceeds an upper bound, and randomly actuated otherwise.

The optimization problem for the MPC controller is then formulated as
\begin{align*}
\min_{\bm{u}_1,\bm{\epsilon}} \quad 
& \sum_{k=0}^{N-1} 
\left(
u_{1,k}
+ \theta_2 \epsilon_k
+ \theta_3 (u_{1,k} - u_{1,k-1})^2
\right) \\
\text{s.t.} \quad
& y_{k+1}
= \bm{a}^\top \bm{y}_{k-10:k}
+ \sum_{i=1}^{3}
\bm{b}_i^\top \bm{u}_{i,k-10:k},
\\
& \theta_1 - \epsilon_k
\le y_{k+1}
\le 26 + \epsilon_k,
\qquad k = 0,\dots,N-1,
\\
& u_{1,k} \in [0,1],
\qquad
\epsilon_k \ge 0,
\qquad k = 0,\dots,N-1.
\end{align*}

where $N = 64$ is the horizon, $\epsilon_k$ is the slack variable to ensure feasibility. The MPC controller is parameterized by $\bm{\theta} = (\theta_1, \theta_2, \theta_3)$, corresponding to the comfort bound, slack penalty, and smoothness penalty. The controller operates with a sampling time of $900$\,s, meaning that the control input is updated every $15$ minutes. We evaluate the controller on a daily basis. For a fixed $\bm{\theta}$, the controller regulates the indoor temperature over a 24-hour window, producing the heating trajectory $Q(t)$ and the indoor temperature trajectory $y(t)$, where $t$ denotes the time of day in seconds.

\paragraph{Objective.}
The goal is to optimize the controller parameters by minimizing the daily cost
\[
\mathcal{J}_{\tau}(\bm{\theta}) 
= \int Q(t) p_{\text{heating},\tau}(t)\, dt 
+ 0.1 \int Q(t) p_{\text{CO}_2}(t)\, dt 
+ 100 \left( \frac{1}{|\mathcal{T}|} \int y(t)\, dt - 24 \right)^2,
\]
where $\tau$ denotes the day index. These three components, which are nonlinear black-box functions with respect to $\bm{\theta}$, capture energy cost, environmental impact, and thermal discomfort, respectively. Let $f(\bm{\theta}) = (Q(t), y(t))$ represent the vector-valued system output. The first two terms correspond to linear functionals of $f(\bm{\theta})$, while the third term is a nonlinear functional of the temperature trajectory. The heating price $p_{\text{heating},\tau}(t)$ varies across days following a sinusoidal pattern between $0$ and $0.232$ EUR/kWh, making the objective $\mathcal{J}_{\tau}(\bm{\theta})$ inherently time-varying.

\paragraph{Learning setup.}
The mapping from $\bm{\theta}$ to trajectories $Q(t)$ and $y(t)$ is generally unknown. In vvBO, we model such mapping using a vector-valued RKHS model with separable kernel, i.e., $K = G\otimes B$. Specifically, we choose the scalar kernel $G$ as a radial basis function kernel and set $B$ to be the identity operator on $\mathcal{Y}$. We assume the full trajectories $f(\bm{\theta})$ are available for measurement.  For numerical implementation, trajectories are discretized over time and represented using the scalar-valued KRR estimator as discussed in Appendix \ref{appendix_output_representation}. The optimization proceeds sequentially. At the end of each day, the measurements are added to the dataset, and the model is updated. Given the forecasted prices for the next day, the algorithm selects parameters $\bm{\theta}$ that minimize $\mathcal{J}_{\tau}(\bm{\theta})$.

\paragraph{Measurements for baseline methods.}
All methods optimize the same time-varying objective $\mathcal J_\tau(\bm\theta)$, but differ in the measurements available during learning. vvBO models the mapping from $\bm\theta$ to the full trajectories $f(\bm\theta)=(Q(t),y(t))$, and evaluates the objective through linear and nonlinear functionals of these trajectories. MTBO assumes access only to finite-dimensional task measurements, which correspond to energy consumption, environmental impact, and thermal discomfort. The surrogate objective is the weighted combination of these three measurements. BO directly observes the scalar objective value $\mathcal J_\tau(\bm\theta)$. CTBO and FFBO treat the price signals as contextual variables and learn scalar surrogate models for the objective. CTBO allows the context to vary across iterations, while FFBO assumes a fixed context during learning. Table~\ref{tab:boptest} summarizes the measurements available to each method.

\begin{table}[t]
\centering
\caption{Measurements available to each method.}
\label{tab:boptest}
\begin{tabular}{lcc}
\toprule
Method & Input & Measurements \\
\midrule

vvBO 
& $\bm{\theta}$ 
& $Q(t),\; y(t)$ \\

MTBO 
& $\bm{\theta}$ 
& energy cost, $\mathrm{CO}_2$ cost, and thermal discomfort \\

BO 
& $\bm{\theta}$ 
& $\mathcal{J}_{\tau}(\bm{\theta})$ \\

CTBO 
& $(\bm{\theta},\; p_{\text{heating},\tau}(t),\; p_{\text{CO}_2}(t))$ 
& $\mathcal{J}_{\tau}(\bm{\theta})$ \\

FFBO 
& $(\bm{\theta},\; p_{\text{heating},0}(t),\; p_{\text{CO}_2}(t))$ 
& $\mathcal{J}_{\tau}(\bm{\theta})$ \\

\bottomrule
\end{tabular}
\end{table}

\paragraph{Additional results.}

\begin{figure}[t]
\centering
\includegraphics[width=\linewidth]{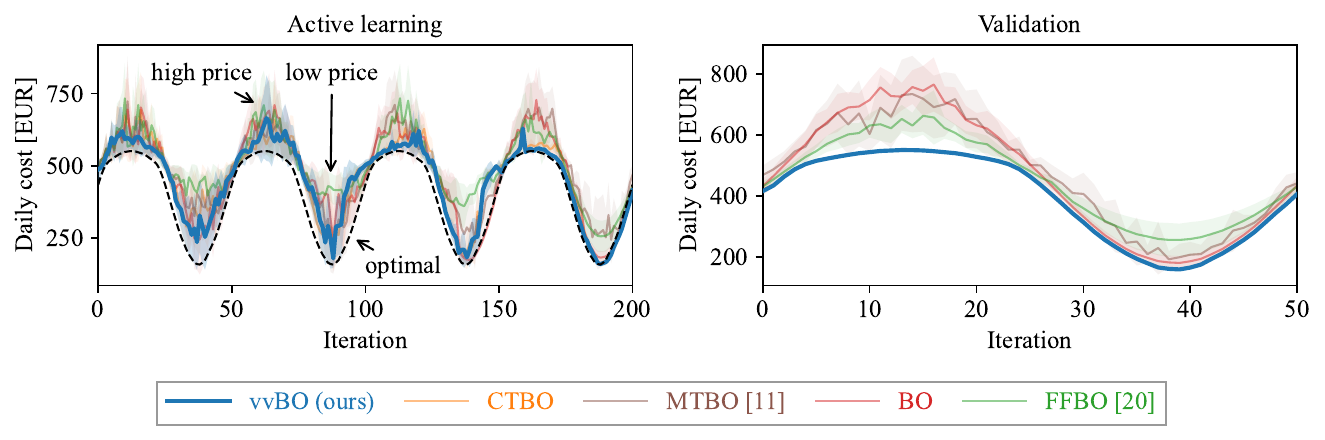}
\caption{Daily cost during learning and validation phases.}
\label{fig:boptest_appendix}
\end{figure}
Fig.~\ref{fig:boptest_appendix} shows the daily cost during both the learning and validation phases. In the first plot, we report the daily cost during the training phase. The black dashed line indicates the optimal cost for each day. After an initial exploration period, vvBO adjusts effectively to the time-varying objective (around iteration 150). In particular, when the heating price is high, vvBO selects parameters that reduce energy consumption by lowering the indoor temperature. Conversely, when the price is low, it increases the temperature to improve thermal comfort, thereby reducing the overall cost. In contrast, most baseline methods fail to track this optimal policy and exhibit significantly higher costs. The second plot shows the validation performance of the learned models. In this phase, no additional data are collected. Each method recommends parameters based on the trained model. vvBO consistently achieves the lowest cost across all days, demonstrating its ability to generalize and adjust to time-varying conditions. This highlights the advantage of leveraging structured measurements for efficient learning in real-world control problems.


\end{document}